\documentclass[twoside]{article}

 \usepackage[accepted]{template}
\usepackage[round]{natbib}

\usepackage{amsmath}
\usepackage{amsfonts}
\usepackage{amssymb}
\usepackage{amsthm}
\usepackage{bm}
\usepackage{bbm}
\usepackage{mathtools}
\usepackage{enumitem}
\usepackage{thmtools,thm-restate}
\usepackage{algorithm}
\usepackage{algorithmic}
\usepackage[capitalise]{cleveref}
\usepackage{color}
\usepackage[dvipsnames]{xcolor}
\usepackage{graphicx}
\usepackage{comment}
\theoremstyle{plain}
\newtheorem{lemma}{Lemma}%[section]
\newtheorem{theorem}{Theorem}%[section]
\newtheorem{proposition}{Proposition}%[section]
\newtheorem{corollary}{Corollary}%[section]

\theoremstyle{definition}
\newtheorem{definition}{Definition}%[section]
\theoremstyle{remark}
\newcommand{\red}[1]{\textcolor{red}{#1}}

\def\AA{\mathcal{A}}

\def\PP{\mathcal{P}}\def\RR{\mathcal{R}}
\def\SS{\mathcal{S}}\def\UU{\mathcal{U}}
\def\VV{\mathcal{V}}\def\XX{\mathcal{X}}

\def\Ebb{\mathbb{E}}

\def\Nbb{\mathbb{N}}
\def\Rbb{\mathbb{R}}

\def\R{\Rbb}\def\N{\Nbb}

\def\t{\top}
\def\*{\star}

\newcommand{\norm}[1]{ \| #1 \|  }

\newcommand{\lr}[2]{ \left\langle #1, #2 \right\rangle}
\newcommand{\<}{\langle}
\renewcommand{\>}{\rangle}
\DeclareMathOperator*{\argmin}{arg\,min}
\DeclareMathOperator*{\argmax}{arg\,max}

\def\regret{\textrm{Regret}}

\newcommand{\E}{\Ebb}

\def\path2var{V^{\textrm{path}^2}}

\def\VI{\textrm{VI}}
\def\DVI{\textrm{DVI}}
\def\EP{\textrm{EP}}
\def\DEP{\textrm{DEP}}
\def\FP{\textrm{FP}}

\newcommand{\since}[1]{\text{($\because$ #1)}}

\usepackage[normalem]{ulem}
\usepackage{xspace}

\def\regular{continuous\xspace}

\def\rregular{regular\xspace} % for (alpha,beta)-regular

\def\rregularity{regularity\xspace}

\renewcommand{\red}{}

\begin{document}

\runningtitle{Online Learning with Continuous Variations}


\twocolumn[

\aistatstitle{Online Learning with Continuous Variations:\\
		Dynamic Regret and Reductions}

\aistatsauthor{ Ching-An Cheng$^*$ \And Jonathan Lee$^*$ \And  Ken Goldberg \And Byron Boots }
\aistatsaddress{ Georgia Tech \And  UC Berkeley \And UC Berkeley \And  Georgia Tech } ]

%\aistatsauthor{Anonymous Authors}

%\aistatsaddress{Anonymous Institutions} ]

\begin{abstract}
	Online learning is a powerful tool for analyzing iterative algorithms. However, the classic adversarial setup fails to capture regularity  that can exist in practice.
	Motivated by this observation, we establish a new setup, called Continuous Online Learning (COL), where the gradient of online loss function changes continuously across rounds with respect to the learner's decisions.
	We show that COL appropriately describes many interesting applications, from general equilibrium problems (EPs) to optimization in episodic MDPs.
	%
	%In particular, we show monotone EPs admit a reduction, achieving sublinear static regret in COL.
	%
	Using this new setup, we revisit the difficulty of sublinear dynamic regret.
	We prove a fundamental equivalence between achieving sublinear dynamic regret in COL and solving certain EPs.
	With this insight, we offer conditions for efficient algorithms that achieve sublinear dynamic regret, even when the losses are chosen adaptively without any \textit{a priori} variation budget.
	Furthermore, we show for COL a reduction from dynamic regret to both static regret and convergence in the associated EP, allowing us to analyze the dynamic regret of many existing algorithms.

\end{abstract}

\section{INTRODUCTION}

Online learning~\citep{gordon1999regret,zinkevich2003online}, which studies the interactions between a learner (i.e. an algorithm) and an opponent through regret minimization, has proved to be a powerful framework for analyzing and designing iterative algorithms. However, while classic setups focus on bounding the worst case,  many applications are not naturally adversarial. In this work, we aim to bridge this reality gap by establishing a new online learning setup that better captures certain regularity that appears in practical problems.

Formally, an online learning problem repeats the following steps: in round $n$, the learner plays a decision $x_n$ from a decision set $\XX$, the opponent chooses a loss function $l_n:\XX \to \R$ based on the decisions of the learner, and then information about $l_n$ (e.g. $\nabla l_n(x_n)$) is revealed to the learner for making the next decision. This abstract setup~\citep{shalev2012online,hazan2016introduction} studies the \emph{adversarial} setting where $l_n$ can be almost arbitrarily chosen except for minor restrictions like convexity. Often the performance is measured relatively through \textit{static regret},
\begin{align}  \label{eq:static regret}
	\textstyle
	\regret_N^s \coloneqq \sum_{n=1}^{N} l_n(x_n) - \min_{x\in\XX}\sum_{n=1}^{N} l_n(x).
\end{align}
Recently, interest has emerged in algorithms that make decisions that are nearly optimal at every round. The regret is therefore measured on-the-fly and suitably named \textit{dynamic regret},
\begin{align} \label{eq:dynamic regret}
	\textstyle
	\regret_N^d \coloneqq \sum_{n=1}^{N} l_n(x_n) -\sum_{n=1}^{N}  l_n(x_n^*),
\end{align}
where $x_n^* \in \argmin_{x\in \XX} l_n(x)$. As dynamic regret by definition upper bounds static regret, minimizing dynamic regret is a more difficult problem.

%\paragraph{Motivation}
While algorithms with sublinear static regret are well understood, the research on dynamic regret is relatively recent. As dynamic regret grows linearly in the adversarial setup, most papers~\citep{zinkevich2003online,mokhtari2016online,yang2016tracking,
	dixit2019online,besbes2015non,jadbabaie2015online,zhang2017improved}
focus on how dynamic regret depends on certain variations of the loss sequence across rounds (such as the path variation $V_N = \sum_{n = 1}^{N-1}\|x_{n}^* - x_{n+1}^*\|$).
\red{Even if the algorithm does not require knowing the variation, the bound is still written in terms of it.}
While tight bounds %in terms of the dependency on these variations
have been established \citep{yang2016tracking}, their results do not always translate into conditions for achieving sublinear dynamic regret in practice, because the size (i.e. budget) of the variation can be difficult to verify beforehand.
This is especially the case when the opponent is \textit{adaptive}, responding to the learner's decisions at each round.
In these situations, it is unknown if existing results become vacuous or yield sublinear dynamic regret.

%While tight bounds have been established (e.g. \citep{yang2016tracking}) in terms of the dependency on these variations, their results do not always translate into conditions for achieving sublinear dynamic regret in practice, except when the loss sequence is chosen obliviously.
%When the loss can be set \emph{adaptively} based on the learner's decisions (as allowed in the general setup), the size of these variations cannot be verified beforehand and therefore it is unknown if existing results become vacuous or yield sublinear dynamic regret.

Motivated by the use of online learning to analyze iterative algorithms in practice, we consider a new setup we call Continuous Online Learning (COL), which directly models regularity in losses as part of the problem definition, as opposed to the classic adversarial setup that adds ad-hoc budgets.
As we will see, this minor modification changes how regret and feedback interact and makes the quest of seeking sublinear dynamic regret well-defined and interpretable, even for adaptive opponents, without imposing variation budgets.
%We revisit these fundamental questions in online learning with a new setup that we call Continuous Online Learning (COL), which directly models regularity in losses as part of the problem definition, as opposed to starting from the classic adversarial setup and then adding ad-hoc features.
%As we will see, this minor modification changes how regret and feedback interact and makes the quest of seeking sublinear dynamic regret well defined even for adaptive opponents.

\subsection{Definition of COL}

A COL problem is defined as follows. We suppose that the opponent possesses a bifunction $f:(x,x') \mapsto f_{x}(x') \in \R$, for $x,x'\in \XX$, that is \emph{unknown} to the learner.
This bifunction is used by the opponent to determine the per-round losses:
in round $n$, if the learner chooses $x_n$, then the opponent responds with
\begin{align} \label{eq:regular per-round loss}
	l_n(\cdot) = f_{x_n} (\cdot).
\end{align}
Finally, the learner suffers $l_n(x_n)$ and receives feedback about $l_n$.
For $f_{x}(x')$, we treat $x$ as the \emph{query argument}  that proposes a question (i.e. an optimization objective $f_{x}(\cdot)$), and treat $x'$ as the \emph{decision argument} whose performance is evaluated. This bifunction $f$ generally can be defined online as queried, with only the limitation that the same loss function $f_{x}(\cdot)$ must be selected by the opponent whenever the learner plays the same decision $x$. Thus, the opponent can be adaptive, but in response to only the learner's current decision.

In addition to the restriction in \eqref{eq:regular per-round loss}, we impose regularity into $f$ to relate $l_n$ across rounds so that seeking sublinear dynamic regret becomes well defined.\footnote{Otherwise the opponent can define $f_{x} (\cdot)$ pointwise for each $x$ to make $l_n(x_n) - l_n(x_n^*)$ constant.}
\begin{definition} \label{def:regular problems}
	We say an online learning problem is \emph{\regular} if $l_n$ is set as in~\eqref{eq:regular per-round loss} by a bifunction $f$  satisfying, $\forall x' \in \XX$,
%	\begin{enumerate}\vspace{-2.5mm}
%		\item $f_{x}(\cdot)$ is a convex function.
%		\item
		$\nabla f_{x}(x')$ is a continuous map in $x$ \footnote{We define $\nabla f_{x}(x')$ as the derivative with respect to $x'$.}.
%	\end{enumerate}
\end{definition}
The continuity structure in \cref{def:regular problems} and the constraint \eqref{eq:regular per-round loss} in COL limit the degree that losses can vary, making it possible for the learner to partially infer future losses from the past experiences.
%This is similar to the predictable construction in \citep{rakhlin2013online}, which on the other hand builds on top of the classic adversarial setup.

%%At this point, the setup of COL may sound restrictive and not of practical interest. To address this concern, we add an important \emph{nuance} in the relationship between loss and feedback in COL.
%In the classic setup, the feedback is directly determined by the loss $l_n$ and the decision $x_n$, e.g., given as full information (receiving $l_n$ or $\nabla l_n(x_n)$) or bandit (just $l_n(x_n)$). On the contrary, in COL we give the opponent the freedom to add additional stochastic or adversarial component into the feedback; e.g., in first-order feedback, the learner could receive $g_n = \nabla l_n(x_n) + \xi_n$, where $\xi_n$ is a probabilistically bounded and potentially adversarial vector, which can be used to model noise or bias in feedback.
%%
%In other words, COL concerns loss with regularity and delegate the feedback to model adversarial and stochastic situations. This slight change in setup allows us to refine algorithm analyses, especially important in studying dynamic regret.

The continuity may appear to restrict COL to purely deterministic settings, but adversity such as stochasticity can be incorporated via an important {nuance} in the relationship between loss and feedback.
In the classic online learning setting, the adversity is incorporated in the loss: the losses $l_n$ and decisions $x_n$ may themselves be generated adversarially or stochastically and then they directly determine the feedback, e.g., given as full information (receiving $l_n$ or $\nabla l_n(x_n)$) or bandit (just $l_n(x_n)$). The (expected) regret is then measured with respect to these intrinsically adversarial losses $l_n$.
By contrast, in COL, we always measure regret with respect to the true underlying bifunction $l_n = f_{x_n}$.
However, we give the opponent the freedom to add an additional stochastic or adversarial component into the feedback;
e.g., in first-order feedback, the learner could receive $g_n = \nabla l_n(x_n) + \xi_n$, where $\xi_n$ is a probabilistically bounded and potentially adversarial vector, which can be used to model noise or bias in feedback.
%In the classic setup, the feedback is directly determined by the loss $l_n$ and the decision $x_n$, e.g., given as full information (receiving $l_n$ or $\nabla l_n(x_n)$) or bandit (just $l_n(x_n)$). On the contrary, in COL we give the opponent the freedom to add additional stochastic or adversarial component into the feedback; e.g., in first-order feedback, the learner could receive $g_n = \nabla l_n(x_n) + \xi_n$, where $\xi_n$ is a probabilistically bounded and potentially adversarial vector, which can be used to model noise or bias in feedback.
%
%In other words, COL concerns loss with regularity and delegate the feedback to model adversarial and stochastic situations. This slight change in setup allows us to refine algorithm analyses, especially important in studying dynamic regret.
In other words, the COL setting models a true underlying loss with regularity, but allows the adversary to be modeled within the feedback.
This addition is especially important for dynamic regret, as it allows us to always consider regret against the true $f_{x_n}$ while  incorporating the possibility of stochasticity.

%Compared with the previous setups, these \regular problems better capture properties of the practical online decision making. The amount of variations here is {internal}, and the separation between query and decision arguments models the partial feedback property commonly seen in performance measure; e.g., the construction of $f_{x}(\cdot)$ itself may involve simulations of a complex process and therefore is expensive to fully differentiate (i.e. directly optimizing the performance index of a decision through back-propagation is infeasible).

%It turns out that the problem above does not have sufficient structure to yield sublinear dynamic regret, because the opponent can define $f_{x} (\cdot)$ pointwise for each $x$ to make $l_n(x_n) - l_n(x_n^*)$ constant.

\subsection{Examples} \label{sec:examples}

At this point, the setup of COL may sound abstract, but this setting is in fact motivated by a general class of problems and iterative algorithms used in practice, some of which have been previously analyzed in the online learning setting.
Generally, COL describes the trial-and-error principle, which attempts to achieve a difficult objective $f_x(x)$ through iteratively constructing a sequence of simplified and related subproblems $f_{x_n}(x)$, similar to majorize-minimize (MM) algorithms.
%This method of relaxation is similar in essence to majorize-minimize (MM) algorithms.
Our first application of this kind is the use of iterative algorithms in solving (stochastic) equilibrium problems (EPs)~\citep{bianchi1996generalized}. EPs are a well-studied subject in mathematical programming, which includes optimization, saddle-point problems, variational inequality (VI)~\citep{facchinei2007finite}, fixed-point problems (FP), etc. Except for toy cases, these problems usually rely on using iterative algorithms to generate $\epsilon$-approximate solutions; interestingly, these algorithms often resemble known algorithms in online learning, such as mirror descent or Follow-the-Leader (FTL).
In \cref{sec:possibility of sublinear dynamic regret,sec:monotone ep as col}, we will show how the residual function of these problems renders a natural choice of bifunction $f$ in COL and how the regret of COL relates to its solution quality. In this example, it is particularly important to classify the adversary (e.g. due to bias or stochasticity) as feedback rather than as a loss function, to properly incorporate the continuity in the source problem.

%An interesting feature among these applications is that the sequence of online loss functions is decided by the learner's decisions in a predictable manner (e.g. through an unknown but fixed Markov decision process (MDP) in IL), as opposed to being selected adversarially.

%Typically in this setting, small changes in the policy $x_n$ give rise to small changes in $l_n$, which suggests a continuous structure in the problem.
%In addition, many well-studied algorithms from a number of other fields can also be framed in this way, including iterative LQR \citep{li2004iterative,liao1991convergence} and smooth convex optimization \citep{nesterov2013introductory}.

Another class of interesting COL problems comes from optimization in episodic Markov decision processes (MDPs). In online imitation learning (IL)~\citep{ross2011reduction}, the learner optimizes a policy to mimic an expert policy $\pi^\star$. In round $n$, the loss is $l_n(\pi) = \E_{s\sim d_{\pi_n}} [c(s,\pi;\pi^\star)]$, where $d_{\pi_n}$ is the state distribution visited by running the learner's policy $\pi_n$ in the MDP, and $c(s,\pi;\pi^\star)$ is a cost that measures the difference between a policy $\pi$ and the expert $\pi^\star$. This is a bifunction form where continuity exists due to expectation and feedback is noisy about $l_n$ (allowed by our feedback model). In fact, online IL is the main inspiration behind this research.
An early analysis of IL was framed using the adversarial, static regret setup~\citep{ross2011reduction}. Recently, results were refined through the use of continuity in the bifunction and dynamic regret \citep{cheng2018convergence,lee2018dynamic,cheng2018accelerating}. This problem again highlights the importance of treating stochasticity as the feedback. We wish to measure regret with respect to the expected cost $l_n(\pi)$ which admits a continuous structure, but feedback only arrives via stochastic samples from the MDP.
Structural prediction and system identification can be framed similarly \citep{ross2012agnostic,venkatraman2015improving}. Details, including new insights into the IL, can be found in \cref{sec:IL}.

Lastly, we note that the classic fitted Q-iteration \citep{gordon1995stable,riedmiller2005neural} for reinforcement learning also uses a similar setup. In the $n$th round, the loss can be written as
$l_n(Q) = \E_{s,a\sim \mu_{\pi(Q_n)}} \E_{s'\sim\PP(s,a)}[(Q(s,a)- r(s,a) - \gamma \max_{a'} Q_n(s',a')   )^2]$, where $\mu_{\pi(Q_n)}$ is the state-action distribution\footnote{Or some fixed distribution with sufficient excitation.} induced by running a policy $\pi(Q_n)$ based on the Q-function $Q_n$ of the learner, and $\PP$ is the transition dynamics, $r$ is the reward, and $\gamma$ is the discount factor. Again this is a COL problem.

\subsection{Main Results}

The goal of this paper is to establish COL and to study, particularly, conditions and efficient algorithms for achieving sublinear dynamic regret. We choose not to pursue algorithms with fast static regret rates in COL, as there have been studies on how algorithms can systematically leverage continuity in COL to accelerate learning~\citep{cheng2018accelerating,cheng2018predictor} although they are framed as online IL research. Knowledge of dynamic regret is less well-known, with the exception of \textbf{} \citet{cheng2018convergence,lee2018dynamic} (both also framed as online IL), which study the convergence of FTL and mirror descent, respectively.

Our first result shows that achieving sublinear dynamic regret in COL is equivalent to solving certain EP, VI, and FP problems that are known to be PPAD-complete\footnote{In short, they are NP problems whose solutions are known to exist, but it is open as to if they belong to P.}~\citep{daskalakis2009complexity}.
In other words, we show that achieving sublinear dynamic regret that is polynomial in the dimension of the decision set can be extremely difficult.

Nevertheless, based on the solution concept of EP, VI, and FP, we show a reduction from monotone EPs to COL, and we present necessary conditions and sufficient conditions for achieving sublinear dynamic regret with polynomial dependency.
Particularly, we show a \emph{reduction} from sublinear dynamic regret to static regret and convergence to the solution of the EP/VI/FP.
This reduction allows us to quickly derive non-asymptotic dynamic regret bounds of popular online learning algorithms based on their known static regret rates.
Finally, we extend COL to consider partially adversarial loss and discuss open questions.

\section{RELATED WORK}

Much work in dynamic regret has focused on improving rates with respect to various measures of the loss sequence's variation.
%\cite{zinkevich2003online} initially showed that the dynamic regret of gradient descent with convex losses is in $O(\sqrt N(1 + V_N))$. When the losses are strongly convex, \cite{mokhtari2016online} showed that this bound becomes $O(1 + V_N)$.
\cite{zinkevich2003online,mokhtari2016online} showed the dynamic regret of gradient descent in terms of the path variation.
Other measures of variation such as functional variation~\citep{besbes2015non} and squared path variation~\citep{zhang2017improved} have also been studied.
While these algorithms may not need to know the variation size beforehand, their guarantees are still stated in terms of these variations.
Therefore, these results can be difficult to interpret when the losses can be chosen adaptively. % based on the learner's decisions.
%, though \citep{yang2016tracking} showed that dynamic regret rates are essentially dependent on these variations. % (i.e. if the path variation grows at a sufficiently large order, there is no hope of achieving sublinear regret.)

\red{To illustrate, consider the online IL problem.
It is impossible to know the variation budget \textit{a priori} because the loss observed at each round of IL is a function of the policy selected by the algorithm.
This budget could easily be linear, if an algorithm selects very disparate policies, or it could be zero if the algorithm always naively returns the same policy.
Thus, existing budget-based results cannot describe the convergence of an IL algorithm.}
%In this paper, we strive for dynamic regret rates that are \textit{interpretable} and ideally sublinear as a function of an order of $N$ alone.
%In the online IL problem, this would tell us exactly how fast the regret grows.}

Our work is also closely related to that of \citet{rakhlin2013online,hall2013dynamical}, which consider \textit{predictable} loss sequences, i.e. sequences that are presumed to be non-adversarial and admit improved regret rates. The former considers static regret for both full and partial information cases, %that are at worst equivalent to usual static regret bounds even if the model is a poor predictor.
and the latter considers a similar problem setting but for the dynamic regret case. These analyses, however, still require a known variation quantity in order to be interpretable. %, but this quantity in practice may be dependent on the online decisions of the learner.

\red{By contrast, we leverage extra structures of COL to provide interpretable dynamic regret rates, without \textit{a priori} constraints on the variation.}
That is, our rates are internally governed by the algorithms, rather than externally dictated by a variation budget.
%need not be explicitly dependent on a variation quantity.
This problem setup is in some sense more difficult, %than the previous problems
as achieving sublinear dynamic regret  requires that both the per-round losses and the loss variation, as a function of the learner's decisions, be \emph{simultaneously} small.
%Furthermore, since the losses in COL are not completely adversarial as in~\citep{rakhlin2013online,hall2013dynamical}, sublinear regret rates are possible.
Nonetheless, %leveraging the fact that the loss sequence in COL is predictable and properties of the corresponding EP,
we can show conditions for sublinear dynamic regret using the bifunction structure in COL.

\section{PRELIMINARIES}

We review background, in particular VIs and EPs, for completeness %Please refer to
~\citep{facchinei2007finite,bianchi1996generalized,konnov2002theory}. % and therein for details and extensions.

%
%\subsection{Notation}

%  and let $\XX$ be a compact convex subset in $\R^d$, which we call the \emph{decision set}.

\paragraph{Notation} Throughout the paper, we reserve the notation $f$ to denote the bifunction that defines COL problems, and we assume $\XX \subset \R^d$ is compact and convex, where $d \in \N_+$ is finite. We equip $\XX$ with norm $\norm{\cdot}$, which is not necessarily Euclidean, and write $\norm{\cdot}_*$ to denote its dual norm. We denote its diameter by $D_\XX \coloneqq \max_{x,x' \in \XX} \norm{x-x'}$.

As in the usual online learning, we are particularly interested in the case where $f_x(\cdot)$ is convex and continuous. %We make this assumption in the following.
For simplicity, we will assume all functions are continuously differentiable, except for $f_{x}(x')$ as a function over the querying argument $x$, where $x' \in \XX$.
%We note that our results can be extended to subdifferentials of non-smooth convex functions.
We will use $\nabla$ to denote gradients. In particular, for the bifunction $f$,
we use $\nabla f$ to denote $\nabla f: x \mapsto \nabla f_{x}(x)$ and we recall, in the context of $f$, $\nabla$ is always with respect to the decision argument. Likewise, given $x \in \XX$, we use $\nabla f_{x}$ to denote $\nabla f_{x}(\cdot)$. %: x \mapsto f_{x'}(x)$.
Note that the continuous differentiability of $f_{x'}(\cdot)$ together with the continuity of $\nabla f_{\cdot}(x)$ implies $\nabla f$ is continuous; the analyses below can be extended to the case where $\nabla f_{x'}(\cdot)$ is a subdifferential.\footnote{Our proof can be extended to upper hemicontinuity for set-valued maps, such as subdifferentials.} %} under proper assumptions.}
Finally,  we  assume, $\forall x\in \XX$,  $\norm{\nabla f_x(x)}_* \leq G $ for some $G < \infty$.

%
%\subsection{Convexity}

\paragraph{Convexity} For $\mu \geq 0$, a function $h:\XX \to \R$ is called $\mu$-strongly convex if it satisfies, for all $x,x' \in \XX$,  $h(x') \geq h(x) + \lr{\nabla h(x)}{x'-x} + \frac{\mu}{2}\norm{x-x'}^2$.
If $h$ satisfies above with $\mu=0$, it is called convex. A function $h$ is called pseudo-convex if $\lr{\nabla h(x)}{x'-x} \geq 0 $ implies $h(x') \geq h(x)$. These definitions have a natural inclusion: strongly convex functions are convex; convex functions are pseudo-convex. We say $h$ is $L$-smooth if $\nabla h$ is $L$-Lipschitz continuous, i.e., there is $ L\in [0,\infty)$ such that $\norm{\nabla h(x)- \nabla h(x')}_* \leq L \norm{x-x'}$ for all $x,x'\in \XX$.
Finally, we will use Bregman divergence $B_{R}(x'||x) \coloneqq R(x') - R(x) - \lr{\nabla R(x)}{x' -x}$ to measure the difference between $x,x'\in \XX$, where $R:\XX \to \R$ is a $\mu$-strongly convex function with $\mu>0$; by definition $B_R(\cdot||x)$ is  also $\mu$-strongly convex.

%We will also use the concept of strongly convex sets.
%\begin{definition}
%	Let $\alpha_\XX \geq 0$.
%	A set $\XX$ is called \emph{$\alpha_\XX$-strongly convex} if, for any $x,x' \in \XX$ and  $\lambda \in [0,1]$, it holds that, for all unit vector $v$,
%	$
%	\lambda x + (1-\lambda) x' +  \frac{\alpha_\XX\lambda (1-\lambda)}{2} \norm{x-x'}^2 v \in \XX
%	$.
%\end{definition}
%When $\alpha_\XX =0$, the definition reduces to usual convexity. Also, we see that this definition implies
%$\alpha_\XX \leq \frac{4}{D_\XX}$. In other words, larger sets are less strongly convex. This can also be seen from the lemma below.
%\begin{lemma} {\normalfont\citep[Theorem 12]{journee2010generalized}}
%	Let $f$ be non-negative, $\alpha$-strongly convex, and $\beta$-smooth on a Euclidean space. Then the set $\{ x | f(x ) \leq r\}$ is $\frac{\alpha}{\sqrt{2r \beta}}$-strongly convex.
%\end{lemma}

%
%\subsection{Fixed Point Problems}

\paragraph{Fixed-Point Problems} Let $T: \XX \to 2^\XX$ be a point-to-set map, where $2^\XX$ denotes the power set of $\XX$. A fixed-point problem $\FP(\XX,T)$ aims to find a point $x^\* \in \XX$ such that $x^\* \in T(x^\*)$. Suppose $T$ is $\lambda$-Lipschitz. It is called non-expansive if $\lambda=1$ and $\lambda$-contractive if $\lambda<1$.
%A fundamental result of FPs is Kakutani  fixed-point theorem.
%\begin{lemma}[\citep{kakutani1941generalization}] \label{lm:fixed piont theorem}
%%Let $\XX$ be convex and compact.
%If $T: \XX \to 2^\XX$ is upper hemicontinuous such that $T(x) $ is non-empty and convex $\forall x \in \XX$, then $\exists x^\* \in \XX$ such that $x^\* \in T(x^\*)$.
%\end{lemma}

%
%\subsection{Variational Inequalities}

\paragraph{Variational Inequalities} VIs study equilibriums defined by vector-valued maps. Let $F: \XX \to \R^d$ be a point-to-point map. The problems $\VI(\XX, F)$ and $\DVI(\XX, F)$ aim to find $x^\star \in \XX$ and  $x_\star \in \XX$, respectively, such that the following conditions are satisfied:
\begin{align*}
	\VI:&   \lr{F(x^\*)}{x -x^\*} \geq 0, &\forall x \in \XX \\
	\DVI:& \lr{F(x)}{x -x_\*} \geq 0,  &\forall x \in \XX
\end{align*}
%A problem $\VI(\XX, F)$ aims to find $x^\* \in \XX$ such that $\lr{F(x^\*)}{x -x^\*} \geq 0$, $\forall x \in \XX$,  and its dual $\DVI(\XX, F)$ aims to find $x_\* \in \XX$ such that $\lr{F(x)}{x -x_\*} \geq 0$, $\forall x \in \XX$.
%\begin{align} \label{eq:VI problem}
%\lr{F(x^\*)}{x -x^\*} \geq 0, \qquad \forall x \in \XX,
%\end{align}
%and its dual $\DVI(\XX, F)$ aims to find $x_\* \in \XX$ such that
%\begin{align} \label{eq:DVI problem}
%\lr{F(x)}{x -x_\*} \geq 0, \qquad \forall x \in \XX.
%\end{align}
VIs and DVIs are also known as Stampacchia and Minty VIs, respectively~\citep{facchinei2007finite}. The difficulty of solving VIs depends on the property of $F$.
For $\mu\geq0$, $F$ is called $\mu$-strongly monotone if $\forall x,x'\in \XX$.
%\begin{align} \label{eq:strong monotonicity}
$\lr{F(x) - F(x')}{x-x'} \geq \mu \norm{x-x'}^2$.
%\end{align}
If $F$ satisfies the above with $\mu=0$, $F$ is called monotone. $F$ is called pseudo-monotone if $\lr{F(x')}{x-x'} \geq 0 $ implies  $\lr{F(x)}{x-x'} \geq 0 $ for $x,x'\in \XX$.
It is known that the gradient of a (strongly/pseudo) convex function is (strongly/pseudo) monotone. % map.

VIs are generalizations of FPs. For a point-to-point map $T:\XX\to\XX$, $\FP(\XX,T)$ is equivalent to $\VI(\XX,I - T)$, where $I$ is the identity map. If $T$ is $\lambda$-contractive, then $F$ is $(1-\lambda)$-strongly monotone.

%
%\subsection{Equilibrium Problems}

\paragraph{Equilibrium Problems} EPs further generalize VIs. Let $\Phi: \XX \times \XX \to \R$ be a bifunction such that %$\Phi(x, \cdot)$ is convex and
$\Phi(x,x) \geq 0$. % and $\Phi(\cdot, x)$ is continuous.\footnote{For simplicity, we adopt the simplest setting discussed in~\citep{?} and note that the conditions can be further relaxed to, e.g., upper semi-continuity.}
The problems $\EP(\XX,\Phi)$ and $\DEP(\XX,\Phi)$ aim to find $x^\star, x_\star \in \XX$, respectively, such that
\begin{align*}
	\EP:& \ \Phi(x^\*,x) \geq 0,  &\forall x \in \XX \\
	\DEP:& \ \Phi(x,x_\*) \leq 0,  &\forall x \in \XX.
\end{align*}
%A problem $\EP(\XX,\Phi)$ aims to find $x^\*\in\XX$ such that
%\begin{align} \label{eq:EP problem}
%\Phi(x^\*,x) \geq 0, \qquad \forall x \in \XX,
%\end{align}
%and its dual $\DEP(\XX,\Phi)$ aims to find $x_\*\in\XX$ such that
%\begin{align} \label{eq:DEP problem}
%\Phi(x,x_\*) \leq 0, \qquad \forall x \in \XX.
%\end{align}
By definition, we have $\VI(\XX,F) = \EP(\XX, \Phi)$ if we define $\Phi(x,x') = \lr{F(x)}{x' - x}$.

%; therefore a VI problem is also called a linear EP problem, since $\Phi(x,\cdot)$ is a linear function.
%It can be verified, with $\Phi$ defined above,  $\VI(\XX,F)= \DEP(\XX, \Phi)$.

We can also define monotonicity properties for EPs. For $\mu\geq0$, $\Phi$ is called $\mu$-strongly monotone if  for $\forall x,x'\in \XX$,
%\begin{align} \label{eq:strong monotonicity (EP)}
$\Phi(x,x') + \Phi(x',x) \leq - \mu \norm{x-x'}^2$.
%\end{align}
It is called monotone if it satisfies the above with $\mu=0$. Similarly, $\Phi$ is called pseudo-monotone if $\Phi(x,x') \geq 0 $ implies  $\Phi(x',x) \leq 0 $ for $x,x'\in \XX$. One can verify that these definitions are consistent with the ones for VIs. %, e.g., if $\Phi$ is monotone, it is also pseudo-monotone.

%\subsection{Primal and Dual Solutions}

\paragraph{Primal and Dual Solutions} We establish some basics of the solution concepts of EPs. As VIs are a special case of EPs, % by the above identification,
these results can be applied to  VIs too.
First, we have a basic relationship between the solution sets, $X^\*$ of EP and $X_\*$ of DEP.
\begin{proposition} \label{pr:primal and dual solutions}
	{\normalfont \citep{bianchi1996generalized}}
	If $\Phi$ is pseudo-monotone, $X^\* \subseteq X_\*$.
	If $\Phi(\cdot, x)$  is continuous $\forall x \in \XX$,  $X_\* \subseteq X^\*$.
\end{proposition}
The proposition states that a dual solution is always a primal solution when the problem is continuous, and a primal solution is a dual solution when the problem is pseudo-monotone. Intuitively, we can think of the primal solutions $X^\*$ as {local} solutions and the dual solutions $X_\*$ as {global} solutions. In particular for VIs, if $F$ is a gradient of some, even nonconvex,  function, any solution in $X_\*$ is a global minimum; any local minimum of a pseudo-convex function is a global minimum~\citep{konnov2002theory}.

We note, however, that~\cref{pr:primal and dual solutions} does not directly ensure that the solution sets are non-empty. The existence of primal solutions $X^\*$ has been extensively studied. Here we include a basic result that is sufficient for the scope of our online learning problems with compact and convex $\XX$.
\begin{proposition}  \label{eq:existence of primal solution}
	{\normalfont \citep{bianchi1996generalized}}
	If $\Phi(x,\cdot)$ is convex and $\Phi(\cdot,x)$ is continuous $\forall x \in \XX$,  then $X^\*$ is non-empty.
\end{proposition}

Analogous results have been established for VIs and FPs as well. If  $F$ and $T$ are continuous then solutions exist for both $\VI(\XX, F)$ and $\FP(\XX, T)$, respectively \citep{facchinei2007finite}.
On the contrary, the existence of dual solutions $X_\*$ is mostly based on assumptions. For example, by~\cref{pr:primal and dual solutions}, $X_\*$ is non-empty when the problem is pseudo-monotone.
%and $X^\*$ is non-empty.
Uniqueness can be established with stronger conditions.
%(which can be ensured by~\cref{eq:existence of primal solution}).
%
\begin{proposition}
	{\normalfont \citep{konnov2002theory}} \label{pr:strongly monotone unique}
	If the conditions of~\cref{eq:existence of primal solution} are met and $\Phi$ is strongly monotone, then the solution to $\EP(\XX, \Phi)$ is unique.
\end{proposition}

%\section{Problem Definition}
%\boots{You have a major section called problem definition here, but it is not so easy to reconstruct the actual problem definition. This requires revisiting the introduction, and flipping back and forth a bit to remind myself what is going on... I wonder if it is possible to resummarize briefly here or reorganize so that the problem definition is more self contained}
%We defined COL problems in~\cref{def:regular problems} using a fixed bifunction with continuous gradients.

%\boots{Maybe this who section should be a subsection of the preliminaries, or after 5.1 (in 5.2?) which revisits continuous online learning problems from Def 1. Switching back and forth is confusing. }

%\section{Possibility of Sublinear Dynamic Regret} \label{sec:possibility of sublinear dynamic regret}

\section{EQUIVALENCE AND HARDNESS} \label{sec:possibility of sublinear dynamic regret}

\red{We first ask what extra information the COL formulation entails. We present this result as an equivalence between achieving sublinear dynamic in COL and solving several mathematical programming problems.}
%We first ask how hard COL is, particularly, whether sublinear dynamic regret with polynomial dependency on $d$
%is even possible. %for COL problems in~\cref{def:regular problems}
%It turns out, \emph{in general}, this is difficult, as least as hard as a set of difficult problems known to be PPAD-complete~\citep{daskalakis2009complexity}, even when $f_x(\cdot)$ is convex and continuous $\forall x \in \XX$.
\begin{theorem} \label{th:equivalent problems}
	Let $f$ be given in~\cref{def:regular problems}. Suppose $f_x(\cdot)$ is convex and continuous. The following problems are equivalent:
	\begin{enumerate}
		\item Achieving sublinear dynamic regret w.r.t. $f$.
		\item $\VI(\XX, F)$ where $F(x) = \nabla f_x(x)$.
		\item $\EP(\XX, \Phi)$ where $\Phi(x,x') = f_x(x') - f_x(x)$.
		\item $\FP(\XX, T)$	where $T(x) = \argmin_{x'\in X} f_x(x')$.
	\end{enumerate}
	Therefore, if there is an algorithm that achieves sublinear dynamic regret that in $poly(d)$, then it solves all PPAD problems in polynomial time.
\end{theorem}

%The main message from
\cref{th:equivalent problems} says that, \red{because of the existence of a hidden bifunction}, achieving sublinear dynamic regret is essentially equivalent to finding an equilibrium $x^\* \in X^\*$, in which $X^\*$ denotes the set of solutions of the EP/VI/FP problems in~\cref{th:equivalent problems}.
Therefore, a \emph{necessary} condition for sublinear dynamic regret is that $X^\*$ is non-empty.
Fortunately, this is true for our problem definition by~\cref{eq:existence of primal solution}. % (as $\XX$ is compact and $F = \nabla f$ is continuous).

Moreover, it suggests that extra structure on COL is necessary for algorithms to  achieve sublinear dynamic regret that depends polynomially on $d$ (the dimension of $\XX$).
%\cref{th:equivalent problems} only concerns asymptotic behaviors, we learn that this is PPAD-complete; otherwise, we can use this algorithm to find an approximate solution to any Brouwer's problem in polynomial time.
The requirement of polynomial dependency is important to properly define the problem. Without it, sublinear dynamic regret can be achieved already at least asymptotically, e.g. by simply discretizing $\XX$ (as $\XX$ is compact and $\nabla f $ is continuous) and grid-searching, albeit with an exponentially large constant.
%When the conditions are met,
%\cref{th:equivalent problems} suggests algorithms can achieve sublinear dynamic regret with polynomial dependency, as we later detail in~\cref{sec:algorithms}.

Due to space limitation, we defer the proof of \cref{th:equivalent problems} to \cref{app:proofs of possibility}, along with other proofs for this section.
But we highlight the key idea is to prove that the gap function $\rho(x) \coloneqq f_{x}(x) - \min_{x' \in X}f_{x}(x')$
%\begin{align} \label{eq:gap function}
%\rho(x) \coloneqq f_{x}(x) - \min_{x' \in X}f_{x}(x')
%\end{align}
can be used as a residual function for the above EP/VI/FP in~\cref{th:equivalent problems}. In particular, we note that, for the $\Phi$ in \cref{th:equivalent problems}, $\rho(x)$ is equivalent to a residual function $r_{ep}(x) \coloneqq \max_{x'\in\XX} -\Phi(x,x')$ used in the EP literature.
%That is, $\rho(x)$ is non-negative, computable in polynomial time (it is a convex program), and $\rho(x) = 0 $ if and only if $x \in X^\*$ (because $f_x(\cdot)$ is convex $\forall x \in \XX$).
%Therefore, \cref{th:equivalent problems} holds, once we show that solving one of these problems is equivalent to achieving sublinear dynamic regret.

Below we discuss sufficient conditions on $f$ based on the equivalence between problems in~\cref{th:equivalent problems},
so that the EP/VI/FP in~\cref{th:equivalent problems} becomes better structured and hence allows efficient algorithms.

\subsection{EP and VI Perspectives} \label{sec:EP and VI perspectives}

We first discuss some structures on $f$ such that the VI/EP in~\cref{th:equivalent problems} can be efficiently solved. From the literature, we learn that the existence of dual solutions is a common prerequisite to design efficient algorithms~\citep{konnov2007combined,dang2015convergence,burachik2016projection,lin2018solving}.
For example, convergence guarantees on combined relaxation methods \citep{konnov2007combined} for VIs rely on the assumption that the dual solution set is non-empty.
Here we discuss some sufficient conditions for having a non-empty dual solution set, %for the existence of dual solutions,
which by~\cref{pr:primal and dual solutions} and~\cref{def:regular problems} is a subset of the primal solution set.

By~\cref{pr:primal and dual solutions} and~\ref{eq:existence of primal solution}, a sufficient condition for non-empty $X_\star$ is \textit{pseudo-monotonicity} of $F$ or $\Phi$ (which we recall is a consequence of monotonicity).
For our problem, %we note that
the dual solutions of the EP and VI %in~\cref{th:equivalent problems}
are \emph{different}, while their primal solutions $X^*$ are the same.
\begin{proposition} \label{pr:dual solutions of EP and VI}
	Let $X_{\*}$ and $X_{\*\*}$ be the solutions to
	$\DVI(\XX,F)$ and $\DEP(\XX,\Phi)$, respectively, where $F$ and $\Phi$ are defined in~\cref{th:equivalent problems}. Then $X_{\*\*} \subseteq X_{\*}$. The converse is true if $f_x(\cdot)$ is linear $\forall x \in \XX$.
\end{proposition}
\cref{pr:dual solutions of EP and VI} shows that, for our problem, pseudo-monotonicity of $\Phi$ is stronger than that of $F$. This is intuitive: as the pseudo-monotonicity of $\Phi$ implies that there is $x_\*$ such that $f_x(x_\*) \leq f_x(x)$, i.e. a decision argument that is consistently better than the querying argument under the latter's own question, whereas the pseudo-monotonicity of $F$ merely requires the intersection of the half spaces of $\XX$ cut by $\nabla f_x(x)$ to be non-empty.
Another sufficient assumption for non-empty $X_{\*}$ of VIs  is that $\XX$ is sufficiently strongly convex. % and there is $x^\* \in X^\*$ on the boundary of $\XX$ with non-vanishing $F(x^\*)$.
This condition has recently been used to show fast convergence of mirror descent and conditional gradient descent~\citep{garber2015faster,veliov2017gradient}. We leave this discussion to \cref{app:strongly convex set}.

The above assumptions, however, are sometimes hard to verify for COL. Here we define a subclass of COL and provide constructive (but restrictive) conditions.
\begin{definition} \label{def:alpha-beta regular problems}
	We say a COL problem with $f$ is $(\alpha, \beta)$-{\rregular} if for some $\alpha,\beta \in [0,\infty)$, $\forall x \in \XX$,
	\begin{enumerate}
		\item $f_{x}(\cdot)$ is a  $\alpha$-strongly convex function.
		\item $\nabla f_{\cdot}(x)$ is a $\beta$-Lipschitz continuous map.
	\end{enumerate}
	%	We say a COL problem is $(\alpha, \beta)$-\emph{\rregular} if its bifunction is .
\end{definition}
We call $\beta$ the \emph{regularity} constant; for short, we will also say $\nabla f$ is \emph{$\beta$-\rregular} and $f$ is $(\alpha,\beta)$-\rregular.
We note that $\beta$ is different from the Lipschitz constant of $\nabla f_x(\cdot)$.
The constant $\beta$ defines the degree of online components; in particular, when $\beta=0$ the learning problem becomes offline.
%It is easy to verify that if $f$ and $f'$ are $(\alpha,\beta)$- and $(\alpha',\beta')$-\rregular, $f+f'$ is $(\alpha+\alpha',\beta+\beta')$-\rregular.
%For example, \citep{cheng2018convergence} shows the constant $\beta$ in IL is related to the stability (like mixing rate) of the MDP.
Based on $(\alpha,\beta)$-\rregularity, we have a sufficient condition to monotonicity.
\begin{proposition} \label{pr:beta-alpha strongly monotone}
	%If $\alpha \geq \beta$,
$\nabla f$ is $(\alpha-\beta)$-strongly monotone.
\end{proposition}

\cref{pr:beta-alpha strongly monotone} shows if $\nabla f_x(\cdot)$ does not change too fast with $x$, %(i.e. the variation $\nabla l_n$ is sufficiently slow across rounds compared to its strong convexity modulus)
then $\nabla f$ is strongly monotone in the sense of VI, implying  $X^\star = X_\star$  is equal to a singleton (but not necessarily the existence of $X_{\star\star}$). Strong monotoncity also implies fast linear convergence is possible for deterministic feedback~\citep{facchinei2007finite}.
%which in turn implies fast linear convergence is possible.
%Furthermore, strongly monotone problems have unique solutions.
When $\alpha =\beta$, it implies at least monotonicity, by which we know $X_\*$ is non-empty.

We emphasize that the condition $\alpha\geq\beta$ is not necessary for monotonicity. The  monotonicity condition of $\nabla f$ more precisely results from the monotonicity
of $\nabla f_{\cdot}(x')$ and $\nabla f_{x}(\cdot)$, as %, for $x,x' \in \XX$,
{\small
$
\lr{\nabla f_x(x) - \nabla f_{x'}(x')}{x-x'}  = \lr{\nabla f_x(x) - \nabla f_x(x')}{x-x'}  + \lr{\nabla f_x(x') - \nabla f_{x'}(x')}{x-y}
$.
}
From this decomposition, we can observe that as long as the sum of $\nabla f_{\cdot}(x')$ and $\nabla f_{x}(\cdot)$ is monotone for any $x,x' \in \XX$, then $\nabla f$ is monotone. In the definition of $(\alpha,\beta)$-\rregular problems, no condition is imposed on $\nabla f_{\cdot}(x)$, so we need $\alpha \geq \beta$ in~\cref{pr:beta-alpha strongly monotone}.

\subsection{Fixed-point Perspective} \label{sec:FP perspective}

We can also study the feasibility of sublinear dynamic regret from the perspective of the FP in~\cref{th:equivalent problems}. Here again we consider $(\alpha,\beta)$-\rregular problems.
%\begin{lemma} \label{lm:Lipschitz continuity of argmin map}
%	Suppose $f$ is $(\alpha,\beta)$-\rregular with $\alpha > 0$. Then $T$ in~\cref{th:equivalent problems} is point-valued and $\frac{\beta}{\alpha}$-Lipschitz.
%\end{lemma}
%%The Lipschitz properties directly implies conditions for nice properties of $T$. The fact that $T$ is continuous and $\XX$ is compact verifies again that a solution to $\FP(\XX, T)$ exists.
%%Furthermore, using the same conditions as in~\cref{pr:beta-alpha strongly monotone}, feasibility of sublinear dynamic regret is apparent in the following proposition.
%%We notice that the conditions below are the same as the one in~\cref{pr:beta-alpha strongly monotone}.
%This immediately implies the following.
\begin{proposition} \label{pr:contraction condition}
	Let $\alpha > 0$. If $\alpha > \beta$, then $T$ is $\frac{\beta}{\alpha}$-contractive; if $\alpha = \beta$, $T$ is non-expansive.
\end{proposition}
We see again that the ratio $\frac{\beta}{\alpha}$ plays an important role in rating the difficulty of the problem. %, similar to $\gamma$ in the fitted Q-iteration example in \cref{sec:examples}.
When $\alpha > \beta$, an efficient algorithm for obtaining the the fixed point solution is readily available (i.e. by contraction)
%if $l_n$ is given as a function in every round.
An alternative interpretation is that $x_n^*$ changes at a slower rate than $x_n$ when $\alpha > \beta$ with respect to $\|\cdot\|$.

%\begin{proof}
%\lee{Maybe move to appendix or omit because it follows from the definition. \\}
%From $\frac{\beta}{\alpha}$-Lipschitz continuity of $T$ in Lemma \ref{lm:Lipschitz continuity of argmin map}, $\|T(x) - T(y)\| \leq \frac{\beta}{\alpha}\|x-y\|$ for $x, y\in \XX$. When $\alpha > \beta > 0$, $\frac{\beta}{\alpha} \in (0, 1)$, so $T$ is a contraction by definition. When $\alpha = \beta > 0$, $T$ is non-expansive by definition.
%\end{proof}

%The above proposition requires $\alpha>0$ to make the $\argmin$ map Lipschitz. For $f$ with $\alpha$, we can regularize it with a Bregman divergence.
%\begin{corollary} \label{cr:Lipschitz continuity of regularized argmin map}
%	Suppose $f$ is $(\alpha,\beta)$-\rregular. Let $R$ be a $\alpha'$-strongly convex and $\beta'$-smooth function, where $\alpha' > 0$.
%	Define $\hat{f}_{x'}(x) = f_{x'}(x) + B_R(x'||x)$. Then $\hat{T}(x) = \argmin_{x' \in \XX} \hat{f}_{x}(x')$ is $\frac{\beta+\beta'}{\alpha+\alpha'}$ Lipschitz, and $\nabla f_x(x) = \nabla \hat{f}_x(x)$.
%\end{corollary}

\section{MONOTONE EP AS COL} \label{sec:monotone ep as col}

After understanding the structures that determine the difficulty of COL,
we describe a converse result of \cref{th:equivalent problems}, which converts monotone EPs into COL. Here we assume that $\Phi(x,\cdot)$ is convex.
\begin{theorem} \label{th:monotone EP as COL}
Let $\EP(\XX,\Phi)$ be monotone with $\Phi(x,x)=0$.\footnotemark Consider COL with $f_x(x') = \Phi(x,x')$. Let $\{x_n\}_{n=1}^N$ be any sequence of decisions and define $\hat{x}_N \coloneqq \frac{1}{N}\sum_{n=1}^{N} x_n$
It holds that
$r_{dep}(\hat{x}_N) \leq \frac{1}{N}\regret_N^s$, where $r_{dep}(x') \coloneqq \max_{x\in\XX} \Phi(x,x')$ is the dual residual.
% The same holds for the best decision in $\{x_n\}_{n=1}^N$.
\end{theorem}
\footnotetext{$\Phi(x,x)=0$ is not a restriction; see \cref{app:reduction from EP to COL}.}
\cref{th:monotone EP as COL} shows monotone EPs can be solved by achieving sublinear static regret in COL, at least in terms of the dual residual. Below we relate bounds on the dual residual back to the primal residual, which we recall is given as $r_{ep}(x) \coloneqq \max_{x'\in\XX} -\Phi(x,x')$.
%(the definition above agrees with \cref{th:equivalent problems} as $\Phi(x,x)=0$)
\begin{theorem} \label{th:from duel residual to primal residual}
Suppose $\Phi(\cdot, x)$ is $L$-Lipschitz, $\forall x\in\XX$. If $\Phi$ satisfies $\Phi(x,x')=-\Phi(x',x)$, i.e. $\Phi$ is skew-symmetric, then $r_{ep}(x) = r_{dep}(x)$.
Otherwise,
\begin{enumerate}
\item For $x\in\XX$ such that $r_{dep} (x) \leq 2 L D_{\XX}$, it holds $r_{ep}(x) \leq 2 \sqrt{2 LD_{\XX}} \sqrt{r_{dep}(x)}$.
\item If $\Phi(x,\cdot)$ is in addition $\mu$-strongly convex with $\mu > 0$, for $x\in\XX$ such that $r_{dep} (x) \leq L^2/\mu$, it holds $r_{ep}(x)  \leq 2.8  (L^2/\mu)^{1/3}  r_{dep}(x)^{2/3}$
\end{enumerate}
\end{theorem}
We can view the above results as a generalization of the classic reduction from convex optimization and Blackwell approachability to no-regret learning~\citep{abernethy2011blackwell}. Generally, the rate of primal residual converges slower than the dual residual. However, when the problem is skew-symmetric (which is true for EPs coming from optimization and saddle-point problems; see \cref{app:reduction from EP to COL}), we recover the classic results.
In this case,
%because $r_{dep}(x) \leq r_{ep}(x)$ (due to monotonicity) and $\regret_N^d = \sum_{n=1}^N r_{ep}(x_n)$ (by \cref{th:equivalent problems}),
%In other words, the results here imply that for these COL problems, minimizing dynamic regret can be reduced to minimizing static regret.
we can show
$
r_{ep}(\hat{x}_N) = r_{dep}(\hat{x}_N) \leq \frac{1}{N}\regret_N^s\leq \frac{1}{N}\regret_N^d  = \frac{1}{N} \sum_{n=1}^N r_{ep}(x_n)
$.

These results complement the discussion in \cref{sec:EP and VI perspectives}, as monotonicity implies the dual solution set $X_{\*\*}$ is non-empty.
Namely, these monotone EPs constitute a class of source problems of COL for which efficient algorithms are available.
%We can also view \cref{th:monotone EP as COL} as a constructive way to come up with COL problems
Proofs and further discussions of this reduction are given in \cref{app:reduction from EP to COL}.

\section{REDUCTION BY REGULARITY} \label{sec:reductions}

Inspired by Theorem~\ref{th:equivalent problems}, we present a reduction from minimizing dynamic regret to minimizing static regret and convergence to $X^\*$. Intuitively, this is possible, because Theorem~\ref{th:equivalent problems} suggests achieving sublinear dynamic regret should not be harder than finding $x^\* \in X^\*$. Define $\regret_N^s(x^\star) \coloneqq \sum_{n = 1}^N l_n(x_n) - l_n(x^\star) \leq \regret_N^s$.
%\cref{th:reduction of dynamic regret} also shows that when the dual solution set $X_\*$ of the EP problem in~\cref{th:equivalent problems} is non-empty, the dynamic regret has at least the same rate as convergence rate to $X_\*$.
\begin{theorem} \label{th:reduction of dynamic regret}
	Let $x^\* \in X^\*$ and $\Delta_n \coloneqq \norm{x_n - x^\*}$. If $f$ is $(\alpha, \beta)$-\rregular for $\alpha,\beta \in [0,\infty)$, then for all $N$,
	\begin{align*}
		\regret_N^d &\leq \min\{\textstyle G\sum_{n=1}^{N}\Delta_n, \regret_N^s(x^\star)\} \\
		&\quad + \textstyle \sum_{n=1}^{N} \min\{\beta D_\XX\Delta_n, \frac{\beta^2}{2\alpha}\Delta_n^2\}
	\end{align*}
	If further $X_{\*\*} $ of the dual EP  is non-empty, $\textstyle
	\regret_N^d \geq \frac{\alpha}{2} \sum_{n=1}^{N} \norm{ x_n^* - x_\*}^2$,
	%\begin{align*}
	%\textstyle
	%\regret_N^d \geq \frac{\alpha}{2} \sum_{n=1}^{N} \norm{ x_n^* - x_\*}^2
	%\end{align*}
	where $x_\* \in X_{\*\*}  \subseteq X^\*$.
\end{theorem}
\cref{th:reduction of dynamic regret} roughly shows that when $x^\*$ exists (e.g. given by the sufficient conditions in the previous section), it provides a stabilizing effect to
the problem, so the dynamic regret behaves almost like the static regret when the decisions are around $x^\*$.

This relationship can be used as a powerful tool for understanding the dynamic regret of existing algorithms designed for EPs, VIs, and FPs. These include, e.g., mirror descent~\citep{beck2003mirror}, mirror-prox~\citep{nemirovski2004prox,juditsky2011solving}, conditional gradient descent~\citep{jaggi2013revisiting}, Mann iteration~\citep{mann1953mean}, etc. Interestingly, many of those are also standard tools in online learning, with static regret bounds that are well known~\citep{hazan2016introduction}.

We can apply \cref{th:reduction of dynamic regret} in different ways, depending on the known convergence of an algorithm. For algorithms whose convergence rate of $\Delta_n$ to zero is known, \cref{th:reduction of dynamic regret} essentially shows that their dynamic regret is at most $O(  \sum_{n=1}^{N}\Delta_n)$.
For the algorithms with only known static regret bounds, we can use a corollary. %, when $\alpha > \beta > 0$.
\begin{corollary} \label{cr:full reduction to static regret}
	If $f$ is $(\alpha,\beta)$-\rregular and $\alpha > \beta$, it holds that
	$
	\regret_N^d \leq \regret_N^s(x^\star) + \frac{\beta^2 \widetilde{\regret_N^s}(x^\star) }{2\alpha(\alpha-\beta)}
	$,
	where {\scriptsize$\widetilde{\regret_N^s}(x^\star)$} denotes the static regret of the linear online learning problem with  $l_n(x) = \lr{\nabla f_n (x_n)}{x}$.
\end{corollary}
The purpose of~\cref{cr:full reduction to static regret} is not to give a tight bound, but to show that for nicer problems with $\alpha > \beta$, achieving sublinear dynamic regret is not harder than achieving sublinear static regret.
For tighter bounds, we still refer to~\cref{th:reduction of dynamic regret} to leverage the equilibrium convergence.
We note that the results in \cref{sec:monotone ep as col} and here concern different classes of COL in general, because $\alpha > \beta$ does \red{not} necessarily imply the $\EP(\XX,\Phi)$ is monotone, but only $\VI(\XX,F)$ unless $f_{x}(\cdot)$ is linear.%, by \cref{pr:dual solutions of EP and VI,pr:beta-alpha strongly monotone}.

Finally, we remark \cref{th:reduction of dynamic regret} is directly applicable to expected dynamic regret (the right-hand side of the inequality will be replaced by its expectation) when the learner only has access to stochastic feedback, because the COL setup in non-anticipating.
Similarly, high-probability bounds can be obtained based on martingale convergence theorems, as in~\citep{cesa2004generalization}. In these cases, we note that the regret is defined with respect to $l_n$ in COL, \emph{not} the sampled losses.

\subsection{Example Algorithms} \label{sec:algorithms}

We showcase applications of~\cref{th:reduction of dynamic regret}. These bounds are \emph{non-asymptotic} and depend polynomially on $d$. Also, these algorithms do not need to know $\alpha$ and $\beta$, except to set the stepsize upper bound for first-order methods. Please refer to~\cref{app:proofs of reduction} for the proofs.

\subsubsection{Functional Feedback}

We first consider the simple  greedy update, which sets %when $l_n(\cdot)$ is revealed in each round:
%\begin{align} \label{eq:greedy update}
$x_{n+1} = \argmin_{x\in X} l_n(x)$.
%\end{align}
By~\cref{pr:contraction condition} and~\cref{th:reduction of dynamic regret}, we see that if $\alpha>\beta$, it has $
\regret_N^d = O(1)$. %CD^2(G+\frac{\beta^2}{2\alpha}) $ for some constant $C>0$.
%\end{proposition}
For $\alpha = \beta$, we can use algorithms for non-expansive fixed-point problems~\citep{mann1953mean}.
\begin{proposition} \label{pr:alpha equals beta full information}
	For $\alpha=\beta$, there is an algorithm that achieves sublinear dynamic regret in $poly(d)$.
\end{proposition}

\subsubsection{Exact First-order Feedback}

Next we use the reduction in Theorem \ref{th:reduction of dynamic regret}
to derive dynamic regret bounds for %a classical online learning algorithm,
 mirror descent, under deterministic first-order feedback.
We recall that mirror descent with step size $\eta_n>0$ follows
\begin{align}\label{eq:mirror descent update}
	x_{n+1} = \argmin_{x \in \XX}\  \< \eta_n g_n, x\> + B_R(x \| x_{n}).
\end{align}
where $g_n$ is feedback direction, $B_R$ is a Bregman divergence with respect to some 1-strongly convex function $R$.
%Gradient descent is a special case when $B_R(x \| x_n) = \frac{1}{2}\|x - x_n\|^2$.
Here we assume additionally that $f_x(\cdot)$ is $\gamma$-smooth with $\gamma > 0$ for all $x \in \XX$. %Please see \cref{app:proofs of reduction} for the constants.

%\cheng{Relax the condition of 1-smooth of $R$}
%\begin{lemma}\label{lm:mirror descent contraction}
%If $f$ is $(\alpha,\beta)$-\rregular, $f_x(\cdot)$ is $\gamma$-smooth for all $x \in \XX$, and $R$ is $1$-strongly convex and $L$ smooth, then for the mirror descent algorithm it holds that
%\begin{align*}
%& B_R(x^\star \| x_{n+1}) \leq  \\
%& \left(1 - 2\eta (\alpha - \beta)L^{-1} + \eta^2 (\gamma + \beta)^2 \right)^{\frac{n-1}{2}} B_R(x^\star \| x_{1}).
%\end{align*}
%\end{lemma}

%If $\alpha > \beta$, we can see that the mirror descent algorithm is $(1 - 2\eta (\alpha - \beta) + \eta^2 (\gamma + \beta)^2)$-contractive by choosing $\eta < \frac{2(\alpha - \beta)}{(\gamma + \beta)^2}$. The next proposition follows immediately from combining Lemma \ref{lm:mirror descent contraction} and Theorem \ref{th:reduction of dynamic regret}.

%\begin{proposition}\label{pr:mirror descent dynamic regret}
%If $\alpha > \beta$ and $\eta < \frac{2(\alpha - \beta)}{(\gamma + \beta)^2}$, then for mirror descent, it holds that $\textstyle
%\regret_N^d = O(1)$.
%\end{proposition}

\begin{proposition} \label{pr:mirror descent dynamic regret}
	Let $f$ be $(\alpha, \beta)$-\rregular and $f_x(\cdot)$ be $\gamma$-smooth, $\forall x \in \XX$. Let $R$ be $1$-strongly convex and $L$-smooth. If $\alpha > \beta$, $g_n = \nabla l_{n}(x_n)$, and $\eta_n < \frac{2(\alpha - \beta)}{L(\gamma + \beta)^2}$, then, for some $0 < \nu < 1$,
	%$\textstyle
	%\regret_N^d = O(1)$ for~\eqref{eq:mirror descent update}.
	$\textstyle
	\regret_N^d \leq (G + \beta D_\XX) \sqrt{2 B_R(x^\star \| x_1)} \sum_{n = 1}^N \nu^{n - 1} = O(1)$ for~\eqref{eq:mirror descent update}.
\end{proposition}

\subsubsection{Stochastic \& Adversarial Feedback} \label{sec:stochastic feedback}

We now consider stochastic and adversarial cases in COL. %by using~\cref{cor:predictable current round}.
As discussed, these are directly handled in the feedback, while the (expected) regret is still measured against the true underlying bifunction. Importantly, we make the subtle assumption that bifunction $f$ is fixed before learning. We consider mirror descent in \eqref{eq:mirror descent update} with additive stochastic and adversarial feedback given as $g_n = \nabla l_n (x_n) + \epsilon_n + \xi_n$, where $\epsilon_n \in \R^d$ is zero-mean noise with $\E \left[ \| \epsilon_n \|^2_* \right] < \infty$ and $\xi_n \in \R^d$ is a bounded adversarial bias.
The component $\epsilon_n$ can come from observing a stochastic loss $l_n(x;\zeta_n)$ with random variable $\zeta_n$, when the true loss is $l_n(x) = \E_{\zeta_n} [ l_n(x; \zeta_n) ]$ (i.e. $\nabla l_n(x_n;\zeta_n) = \nabla l_n(x_n) + \epsilon_n$). On the other hand the adversarial component $\xi_n$ can describe extra bias in computation.
We consider the expected dynamic regret
%\begin{align*}
%\textstyle
$\E [\regret_N^d] = \E [ \sum_{n = 1}^N l_n (x_n) - \min_{x \in \XX} l_n (x)  ]
$, where the expectation is over $\epsilon_n$.
%\end{align*}
%where again $x_n^* \in \argmin_{x \in \XX} l_n$ is the minimizer of the true loss function $l_n$.
Define $\Xi \coloneqq \sum_{n = 1}^N \| \xi_n \|_*$. By reduction to static regret
in \cref{cr:full reduction to static regret}, we have the following proposition.

\begin{proposition}\label{pr:stochastic mirror descent}
	If $f$ is fixed before learning, $\alpha > \beta$ and $\eta_n = \frac{1}{\sqrt{n}}$, then mirror descent with $g_n = \nabla l_n(x_n) + \epsilon_n + \xi_n$ has
	$\E [\regret_N^d] = O(\sqrt N + \Xi)$.
\end{proposition}
%The same rate holds with adversarial feedback when $g_n = \nabla l_n (x_n) + \epsilon_n + \xi_n$, where $\xi_n$ is a bounded adversarial bias and $\Xi \coloneqq \sum_{n=1}^{N} \norm{\xi_n}_*$, but a penalty is paid for the amount of bias injected.
% \leq O(\sqrt{N})$.
%\begin{proposition}\label{pr:stochastic adversarial mirror descent}
%	If $\alpha > \beta$ and $\eta_n = \frac{1}{\sqrt{n}}$, mirror descent with $g_n = \nabla l_n(x_n) + \epsilon_n + \xi_n$ has
%	$\E [\regret_N^d] = O(\sqrt N + \Xi)$.
%\end{proposition}

%\cheng{Show the above holds with when the gradient is biased, so long as the biased overall is sublinear.}

\subsection{Remark}

Essentially, our finding indicates that the feasibility of sublinear dynamic regret is related to a problem's properties.
For example, the difficulty of the problem depends largely on the ratio $\frac{\beta}{\alpha}$ when there is no other directional information about $\nabla f_{\cdot}(x)$, such as monotonicity.
%
%We have shown that when $\frac{\beta}{\alpha}<1$, the problem is strongly monotone/contractive, resulting in a natural form of stability.
%%
%When $\frac{\beta}{\alpha}=1$, it becomes only marginally stable, but efficient algorithms are still available. %and the convergence relies on that $\XX$ is bounded. % although the dynamic regret still depends polynomially on $d$.
When $\beta\leq\alpha$, we have shown efficient algorithms are possible. But, for $\beta>\alpha$, we are not aware of any efficient algorithm. If one exists, it would solve all $(\alpha,\beta)$-\rregular problems, which, in turn, would efficiently solve all EP/VI/FP problems as we can formulate them into the problem of solving COL problems with sublinear dynamic regret by~\cref{th:equivalent problems}.
%Precisely, without loss of generality, let us assume $f_n(\cdot)$ is strongly convex\footnote{We can also always define $\hat{f}_{x'}(x) \coloneqq f_{x'}(x) + B_R(x||x')$ which shares the same equilibriums as $f$, and study instead its convergence}. For achieving sublinear dynamic regret, \cref{th:equivalent problems} shows the necessity for the gap function in~\eqref{eq:gap function} to converge to zero, which implies eventually $x_n$ converges to $x_n^*$

\section{EXTENSIONS} \label{sec:extensions}

The COL framework reveals some core properties of dynamic regret. However, while we allow feedback to be adversarial, we still assume that the same loss function $f_x(\cdot)$ must be returned by the bifunction for the same query argument $x \in \XX$.
Therefore, COL does not capture time-varying situations where the opponent's strategy can change across rounds.
Also, this constraint allows the learner to potentially enumerate the opponent.
%That is, if we do not concern the requirement of polynomial dependency on $d$ (e.g. when $d$ is small), then we can always use a grid search algorithm to (asymptotically) achieve sublinear dynamic regret.
Here we relax \eqref{eq:regular per-round loss} and define a generalization of COL. The proofs of this section are included in~\cref{app:proofs of extensions}.
\begin{definition} \label{def:predictable problems}
	We say an online learning problem is \emph{$(\alpha,\beta)$-predictable} with $\alpha,\beta \in [0,\infty)$ if $\forall x \in \XX$,
	\begin{enumerate}
		\item $l_{n}(\cdot)$ is a $\alpha$-strongly convex function.
		\item $\norm{ \nabla l_n (x) - \nabla l_{n-1} (x) }_* \leq \beta \norm{x_n - x_{n-1}} + a_n$, where $a_n \in [0, \infty)$ and  $\sum_{n=1}^{N} a_n = A_N  = o(N)$.

	\end{enumerate}

\end{definition}
This problem generalizse COL along two directions: 1) it makes the problem non-stationary; 2) it allows adversarial components within a sublinear budget inside the loss function. We note that the second condition above is different from having adversarial feedback, e.g., in \cref{sec:stochastic feedback}, because the regret now is measured with respect to the adversarial loss as opposed to those generated by a fixed bifunction. This new condition can make achieving sublinear dynamic regret considerably harder.

Let us further discuss the relationship between $(\alpha,\beta)$-predictable and $(\alpha,\beta)$-\rregular problems.
First, a contraction property like~\cref{pr:contraction condition} still holds.
\begin{proposition}\label{pr:generalized contraction property}
	For $(\alpha,\beta)$-predictable problems with $\alpha>0$,
	$\norm{x_n^* - x_{n-1}^*} \leq \frac{\beta}{\alpha}\norm{x_n - x_{n-1}} + \frac{a_n}{\alpha}$.
\end{proposition}
\cref{pr:generalized contraction property} shows that when functional feedback is available and $\frac{\beta}{\alpha}<1$, sublinear dynamic regret can be achieved, e.g., by a greedy update.
However, one fundamental difference between predictable problems and COL problems is the lack of equilibria $X^*$, which is the foundation of the reduction in~\cref{th:reduction of dynamic regret}. %results based on first-order feedback.
This makes achieving sublinear dynamic regret %in the predictable problem
much harder when functional feedback is unavailable or when $\alpha=\beta$.
Using \cref{pr:generalized contraction property}, we establish some preliminary results below.
%In essence, we find that, to achieve sublinear dynamic regret in predictable problems, mirror descent must maintain a sufficiently large step size, unlike the regular problems which allow for decaying step size.
\begin{theorem} \label{th:predictable problem with }
	%Let $\frac{\beta}{\alpha} < 1$. For $(\alpha,\beta)$-predictable problems, there is $\eta_1,\eta_2 \in (0,\infty)$ such that
	%mirror descent with step size $\eta_n \in [\eta_1, \eta_2]$ achieves $\regret_N^d \leq O(1 +  A_{N} + \sqrt{N A_N})$ with deterministic first-oder feedback.
	Let $\frac{\beta}{\alpha} < \frac{\alpha}{2L^2 \gamma}$. For $(\alpha,\beta)$-predictable problems, if $l_n(\cdot)$ is $\gamma$-smooth and $R$ is $1$-strongly convex and $L$-smooth, then
	mirror descent with deterministic feedback and step size $\eta = \frac{\alpha}{2L\gamma^2}$
	achieves $\regret_N^d = O(1 +  A_{N} + \sqrt{N A_N})$.
	% and achieving $\regret_N^d \leq O(\sqrt{N} + A_{N})$ with stochastic first-order feedback.
\end{theorem}
%Therefore, we do not know whether mirror descent can achieve sublinear dynamic regret in expectation for general predictable problems with $\frac{\beta}{\alpha} < 1$  when only stochastic first-order feedback is available. % .
%Nonetheless, we note that, under a stronger assumption that $ x_n^* \in int(\XX)$, it is possible~\citep{lee2018dynamic}.
We find that, in \cref{th:predictable problem with }, mirror descent must maintain a sufficiently large step size in predictable problems, unlike COL problems which allow for decaying step size.
When $\alpha = \beta$, we can show that sublinear dynamic regret is possible under functional feedback.
\begin{theorem} \label{th:alpha equals beta predictable problem}
	For $\alpha = \beta$, if $A_{\infty} < \infty$ and $\|\cdot\|$ is the Euclidean norm, then there is an algorithm with functional feedback achieving sublinear dynamic regret. For $d=1$ and $a_n = 0$ for all $n$, sublinear dynamic regret is possible regardless of $\alpha, \beta$.
\end{theorem}

%\begin{proof}[Proof Sketch.]
%	We prove that $\|x_n - x_n^*\|$ converges to zero. We reduce this problem into a discrete-time pursuit-evasion game with variable speeds, such that the speed of the evader is at most the speed of the pursuer~\citep{alexander2006pursuit}.
%	For the pursuer, we define a constrained version of the greedy algorithm, which takes half steps: $x_{n+1}  =  \argmin_{x \in \XX}\ l_n(x)$ such that $\|x_{n+1} - x_n\| \leq \frac{1}{2} \|x_n  - x_n^*\|$.
%	We show that if $\lim_{n \rightarrow \infty} \|x_n - x_n^*\| > 0$, then $x_n$ and $x_n^*$ travel in a unbounded in a straight line, contradicting the compactness of $\XX$.
%\end{proof}

We do not know, however, whether sublinear dynamic regret is feasible when $\alpha = \beta $ and $A_{\infty} = \infty$. We conjecture this is infeasible when the feedback is only first-order, as mirror descent is insufficient to solve monotone problems using the last iterate~\citep{facchinei2007finite} which contain COL with $\alpha=\beta$ (a simpler case than predictable online learning with $\alpha=\beta$).

%In other words, we can think of minimizing dynamic regret as a chaser-and-evader game, where the learner is the chaser and the opponent is the evader whose speed is constrained by the speed of the chaser (\cref{pr:contraction condition}). For an $(\alpha,\beta)$-\rregular problem, if there is no other information about $\nabla f_{\cdot}(x)$ (e.g. monotonicity), we have shown in~\cref{pr:beta-alpha strongly monotone} and~\cref{pr:contraction condition} that the possibility of efficient algorithms depend fundamentally on the condition number $\frac{\beta}{\alpha}$.

%rom the chaser-and-evader perspective, it means that the evader's speed is necessary slower than the chaser, so one day the evader will fail to escape. On the critical condition when $\alpha=\beta$, the chaser is allowed to travel at the same speed as the chaser; in this case the convergence depends on the assumption that $\XX$ is bounded so the chaser cannot escape forever.

%\subsection{New Online Learning Setup}
%
%Queries and decisions can be different.
%\cheng{Discuss mirror-prox, average decisions, connection to solving for approximate dual solution.}

%\subsection{Zero-Order problems}

\section{CONCLUSION}

We present COL, a new class of online problems where the gradient varies continuously across rounds with respect to the learner's decisions.
%This setup is motivated by the use of online learning to analyze iterative algorithms.
We show that this setting can be equated with certain equilibrium problems (EPs).
Leveraging this insight, we present a reduction from monotone EPs to COL, and show necessary conditions and sufficient conditions for achieving sublinear dynamic regret. % with polynomial dependency on the problem's dimension.
Furthermore, we show a reduction from dynamic regret to static regret and the convergence to equilibrium points. % in these problems.
%
%These results reveal some core difficulties in achieving sublinear dynamic regret when there is no external constraint on the loss variation.

%Finally we presented a generalization of the \regular problem to predictable problems that incorporate an adversarial component.
There are several directions for future research on this topic.
Our current analyses focus on classical algorithms in online learning. We suspect that the use of adaptive or optimistic methods can accelerate convergence to equilibria, if some coarse model can be estimated.
%We also introduce the predictable online learning setting, which generalizes the \regular problem to incorporate adversarial components.
%
In addition, although we present some preliminary results showing the possibility for interpretable dynamic regret rates in predictable online learning, further refinement and understanding the corresponding lower bounds remain important future work. Finally, while the current formulations restrict the loss to be determined solely by the learner's current decision, extending the discussion to history-dependent bifunctions is an interesting topic.

\subsubsection*{Acknowledgements}
We thank Alekh Agarwal and Chen-Yu Wei for their insightful feedback on the problem setup. We are also grateful to Geoff Gordon and Remi Tachet des Combes for suggestions on refining the connection between COL and monotone EPs.

%Use the unnumbered third level heading for the acknowledgements.  All
%acknowledgements go at the end of the paper.

%\subsubsection*{References}

%References follow the acknowledgements.  Use an unnumbered third level
%heading for the references section.  Any choice of citation style is
%acceptable as long as you are consistent.  Please use the same font
%size for references as for the body of the paper---remember that
%references do not count against your page length total.

%\bibliographystyle{plain}

%\bibliographystyle{plainnat}
\bibliography{../ref}

\clearpage
\appendix
\onecolumn
\allowdisplaybreaks

\section{Complete Proofs of \cref{sec:possibility of sublinear dynamic regret}} \label{app:proofs of possibility}

\subsection{Proof of~\cref{th:equivalent problems}}

\subsubsection{Highlight}

The key idea to proving~\cref{th:equivalent problems} is that the gap function $\rho(x) \coloneqq f_{x}(x) - \min_{x' \in X}f_{x}(x')$
%\begin{align} \label{eq:gap function}
%\rho(x) \coloneqq f_{x}(x) - \min_{x' \in X}f_{x}(x')
%\end{align}
can be used as a residual function for the above EP/VI/FP in~\cref{th:equivalent problems}. That is, $\rho(x)$ is non-negative, computable in polynomial time (it is a convex program), and $\rho(x) = 0 $ if and only if $x \in X^\*$ (because $f_x(\cdot)$ is convex $\forall x \in \XX$).
Therefore, to show~\cref{th:equivalent problems},
we only need to prove that solving one of these problems is equivalent to achieving sublinear dynamic regret.  %The full proof is given in~\cref{app:proofs of possibility}

First, suppose an algorithm generates a sequence $\{x_n \in \XX \}$ such that $\lim_{n\to \infty} x_n = x^\*$, for some $x^\* \in X^\*$. To show this implies $\{x_n \in \XX \}$ has sublinear dynamic regret, we first show
$\lim_{x\to x^\* \in X^\*} \rho(x) = 0$.
Then define $\rho_n = \rho(x_n)$. Because $\lim_{n\to \infty} \rho_n = 0$, we have $\regret_N^d = \sum_{n=1}^N \rho_n= o(N)$.
%Let us and write $\regret_N^d = \sum_{n=1}^N \rho_n$.

Next, we prove the opposite direction. Suppose an algorithm generates a sequence $\{x_n \in \XX \}$ with sublinear dynamic regret. This implies that $\hat{\rho}_N \coloneqq \min_n \rho_n \leq \frac{1}{N} \sum_{n=1}^{N} \rho_n$ is in $o(1)$ and non-increasing. Thus, $\lim_{N\to\infty}\hat{\rho}_N = 0$. As $\rho$ is a proper residual, the algorithm solves the EP/VI/FP problem by returning the decision associated with $\hat{\rho}_N$.

The proof of PPAD-completeness is based on converting the residual of a Brouwer's fixed-point problem %(which is known to be PPAD-complete~\citep{daskalakis2009complexity})
to a bifunction, and use the solution along with $\hat{\rho}_N$ above as the approximate solution.

Note that the gap function $\rho$, despite motivated by dynamic regret here, corresponds to a natural gap function $r_{ep}(x) \coloneqq \max_{x'\in\XX} -\Phi(x,x')$ used in the EP literature, showing again a close connection between the dynamic regret and the EP in~\cref{th:equivalent problems}.
Nonetheless, $\rho(x)$ %in~\eqref{eq:gap function}
is not conventional for VIs and FPs. Below we relate $\rho(x)$ to some standard residuals of VIs and FPs under a stronger assumption on $f$.

\begin{proposition} \label{pr:other conventional residuals}
	For $\epsilon > 0$, consider some $x_\epsilon \in \XX$ such that $\rho(x_\epsilon) \leq \epsilon$. If $f_{x_\epsilon}(\cdot)$ is $\alpha$-strongly convex, then  $\lim_{\epsilon \to 0} \lr{\nabla f_{x_\epsilon}(x_\epsilon)}{x -x_\epsilon} \geq 0$, $\forall x \in \XX$, and $\lim_{\epsilon \to 0} \norm{x_\epsilon - T(x_\epsilon)}  = 0$.
\end{proposition}

%\subsection{Necessary Conditions}
%
%We can use~\cref{th:equivalent problems} to find necessary conditions of
%when sublinear dynamic regret in COL is possible.

%We proved Theorem~\ref{th:equivalent problems} by using

\subsubsection{Full proof}

Now we give the details of the steps above.

We first show the solutions sets of the EP, the VI, and the FP are identical.
% and later we only need to establish the equivalence between sublinear dynamic regret and converging to $X^\*$.
\begin{itemize}
	\item $2. \implies 3.$\\ Let $x_\VI^\star \in \XX$ be a solution to $\VI(\XX, F)$ where $F(x) = \nabla f_x(x)$.
	That is, for all $x \in \XX$, $\<\nabla f_{x_\VI^\star}(x_\VI^\star),x-x_\VI^\star\>\geq0$.
	The sufficient first-order condition for optimality implies that $x_\VI^\star$ is optimal for $f_{x_\VI^\star}$.
	Therefore, $f_{x_\VI^\star}(x_\VI^\star) \leq f_{x_\VI^\star}(x)$ for all $x \in \XX$, meaning that $x_\VI^\star$ is also a solution to $\EP(\XX, \Phi)$ where $\Phi(x, x') = f_x(x') - f_x(x)$.

	%\item $(3. \implies 2.)$: This direction follows in the way as the previous direction but instead using the necessary first-order condition for optimality to show that if $x_{\EP}^\star$ is a solution to $\EP(\XX, \Phi)$ then  $f_{x_\EP^\star}(x_\EP^\star) \leq f_{x_\EP^\star}(x)$ for all $x \in \XX$.
	%This implies $\<\nabla f_{x_\EP^\star}(x_\EP^\star), x - x_\EP^\star\> \geq 0$ for all $x \in \XX$.

	\item $3. \implies 4.$\\
	Let $x_\EP^\star \in \XX$ be a solution to $\EP(\XX, \Phi)$. By definition, it satisfies $f_{x_\EP^\star}(x_\EP^\star) \leq f_{x_\EP^\star}(x)$ for all $x \in \XX$, which implies $x_\EP^\star = \argmin_{x \in \XX} f_{x_\EP^\star}(x) = T(x_\EP^\star)$.
	Therefore, $x_\EP^\star$ is a also solution to $\FP(\XX, T)$, where $T(x') = \argmin_{x \in \XX} f_{x'}(x)$.

	%\item $(4. \implies 3.)$: Again by definition of the solution to $\FP(\XX, T)$, we have $x_\FP^\star = \argmin_{x \in\XX} f_{x_\FP^\star}(x)$, which implies $f_{x_\FP^\star}(x_\FP^\star) \leq f_{x_\FP^\star}(x)$ for all $x \in \XX$.

	\item $4. \implies 2.$\\ If $x_\FP^\star$ is a solution to $\FP(\XX, T)$, then $x_\FP^\star = \argmin_{x \in \XX} f_{x_\FP^\star}(x)$. By the necessary first-order condition for optimality, we have $\<\nabla f_{x_\FP^\star}(x_\FP^\star)x - x_\FP^\star\> \geq 0$ for all $x \in \XX$. Therefore $x_\FP^\star$ is also a solution to $\VI(\XX, F)$ where $F(x) = \nabla f_x(x)$.
\end{itemize}

	Let $X^\*$ denote their common solution sets. To finish the proof of equivalence in \cref{th:equivalent problems}, we only need to show that converging to $X^\*$ is equivalent to achieving sublinear dynamic regret.
	\begin{itemize}
	\item Suppose there is an algorithm that generates a sequence $\{x_n \in \XX \}$ such that $\lim_{n\to \infty} x_n = x^\*$, for some $x^\* \in X^\*$. To show this implies $\{x_n \in \XX \}$ has sublinear dynamic regret, we need a continuity lemma.

	\begin{lemma} \label{lm:continuity lemma}
		$\lim_{x\to x^\* \in X^\*} \rho(x) = 0$.
	\end{lemma}
	\begin{proof}
		Let $\bar{x} \in T(x)$.
		Using convexity, we can derive that
%		\begin{small}
			\begin{align*}
			&\rho(x)  = f_{x}(x) - f_{x}(\bar{x}) \leq \lr{\nabla f_x(x)}{x-\bar{x}} \\
			&\leq \lr{\nabla f_{x^\*}(x^\*)}{x-\bar{x}} + \norm{\nabla f_{x^\*}(x^\*) - \nabla f_x(x)}_*\norm{x-\bar{x}} \\
			&\leq \lr{\nabla f_{x^\*}(x^\*)}{x^\*-\bar{x}} + \norm{\nabla f_{x^\*}(x^\*)}_*\norm{x-x^\*} + \norm{\nabla f_{x^\*}(x^\*) - \nabla f_x(x)}_*\norm{x-\bar{x}}\\
			&\leq \norm{\nabla f_{x^\*}(x^\*)}_*\norm{x-x^\*}+ \norm{\nabla f_{x^\*}(x^\*) - \nabla f_x(x)}_*\norm{x-\bar{x}}
			\end{align*}
%		\end{small}%
		where the second and the third inequalities are due to  Cauchy-Schwarz inequality, and the last inequality is due to that $x^\*$ solves $\VI(\XX,  \nabla f)$. By continuity of $\nabla f$, the above upper bound vanishes as $x \to x^\*$.
	\end{proof}
	For short hand, let us define $\rho_n = \rho(x_n)$; we can then write $\regret_N^d = \sum_{n=1}^N \rho_n$. By~\cref{lm:continuity lemma}, $\lim_{n \to \infty} x= x^\*$ implies that $\lim_{n\to \infty} \rho_n = 0$.
	Finally, we show by contradiction that $\lim_{n\to \infty} \rho_n = 0$ implies $\regret_N^d = o(N)$. Suppose the dynamic regret is linear. Then $c > 0$ exists such that there is a subsequence $\{\rho_{n_i}\}$ satisfying $\rho_{n_i} \geq c $ for all $n_i$. However, this contradicts with $\lim_{n\to \infty} \rho_n = 0$.

	\item
		We can also prove the opposite direction. Suppose an algorithm generates a sequence $\{x_n \in \XX \}$ with sublinear dynamic regret. This implies that $\hat{\rho}_N \coloneqq \min_n \rho_n \leq \frac{1}{N} \sum_{n=1}^{N} \rho_n$ is in $o(N)$ and  non-increasing. Thus, $\lim_{N\to\infty}\hat{\rho}_N = 0$ and the algorithm solves the VI/EP/FP problem because $\rho$ is a residual. Alternatively, we may view $\hat \rho$ as a Lyapunov-like function. The sequence of minimizers $\hat x_N = \argmin_{x_n} \rho(x_n)$ are confined to the level sets of $\rho$, which converge to the zero-level set. Since $\XX$ is compact, $\hat x_N$ converges to this set.
	\end{itemize}

	%Let $\hat{x}_N$ be the corresponding decision of $\hat{\rho}_N$. To finish to proof, we need to show that as $N \to \infty$, $\hat{x}_N$ is also a solution to the aforementioned problems.  We note that this is  non-trivial, as in general $\hat{x}_N$ may not converge.

Finally, we show the PPAD-completeness by proving that achieving sublinear dynamic regret with polynomial dependency on $d$ implies solving a Brouwer's problem (finding a fixed point of a continuous point-to-point map on a convex compact set). Because Brouwer's problem is known to be PPAD-complete~\cite{daskalakis2009complexity}, we can use this algorithm to solve all PPAD problems.

Given a Brouwer's problem on $\XX$  with some continuous map $\hat{T}$. We can define the bifunction $f$ as $f_{x'}(x) = \frac{1}{2}\norm{x - \hat{T}(x')}_2^2$, where $\norm{\cdot}_2$ is Euclidean. Obviously, this $f$ satisfies~\cref{def:regular problems}, and its gap function is zero at $x^\*$ if and only $x^\*$ is a solution to the Brouwer's problem.  Suppose we have an algorithm that achieves sublinear dynamic regret for \regular online learning.  We can use the definition $\hat{\rho}_N$ in the proof above to return a solution whose gap function is less than $\frac{1}{2}\epsilon^2$, which implies an $\epsilon$-approximate solution to Brouwer's problem (i.e. $\norm{x - \hat{T}(x)} \leq \epsilon$). If the dynamic regret depends polynomially on $d$, we have such an $N$ in $poly(d)$, which implies solving any Brouwer's problem in polynomial time.

\subsubsection{Proof of \cref{pr:other conventional residuals}}
	%Let $\epsilon >0$ be arbitrary and let $x_\epsilon \in \XX$ satisfying $\rho(x_\epsilon) \leq \epsilon$; such an $x^*$ exists, as it can be generated by choosing some $N$ large enough such that $\hat{\rho}_N\leq \epsilon$ and setting $x_\epsilon = x_\epsilon_N$. To prove the statement above, it is sufficient to show that as $\epsilon \to 0$, $x_\epsilon$ solves the aforementioned problems.
	For the VI problem, let $x_\epsilon^* = T(x_\epsilon)$ and notice that
	\begin{align} \label{eq:a simple fact of epsilon instant regret}
	\frac{\alpha}{2}\norm{x_\epsilon-x_\epsilon^*}^2\leq f_{x_\epsilon}(x_\epsilon) - f_{x_\epsilon}(x_\epsilon^*) \leq \epsilon
	\end{align}
	for some $\alpha>0$.
	Therefore, for any $x \in \XX$,
%	\begin{small}
		\begin{align*}
		\lr{\nabla f_{x_\epsilon}(x_\epsilon)}{x -x_\epsilon} & \geq \lr{\nabla f_{x_\epsilon}( x_\epsilon^*)}{x -x_\epsilon} -\norm{\nabla f_{x_\epsilon}( x_\epsilon^*)-\nabla f_{x_\epsilon}(x_\epsilon)}_*\norm{x -x_\epsilon} \\
		&\geq \lr{\nabla f_{x_\epsilon}( x_\epsilon^*)}{x - x_\epsilon^*} - \norm{\nabla f_{x_\epsilon}( x_\epsilon^*)}_*\norm{x_\epsilon - x_\epsilon^*}
		 -\norm{\nabla f_{x_\epsilon}( x_\epsilon^*)-\nabla f_{x_\epsilon}(x_\epsilon)}_*\norm{x -x_\epsilon} \\
		&\geq - \norm{\nabla f_{x_\epsilon}( x_\epsilon^*)}_*\norm{x_\epsilon - x_\epsilon^*} -\norm{\nabla f_{x_\epsilon}( x_\epsilon^*)-\nabla f_{x_\epsilon}(x_\epsilon)}_*\norm{x -x_\epsilon}
		\end{align*}
%	\end{small}%
	Since $\norm{x_\epsilon-x_\epsilon^*}^2 \leq \frac{2\epsilon}{\alpha}$, by continuity of $\nabla f_{x_\epsilon}$, it satisfies that $\lim_{\epsilon \to 0} \lr{\nabla f_{x_\epsilon}(x_\epsilon)}{x -x_\epsilon} \geq 0$, $\forall x \in \XX$.

	For the fixed-point problem, similarly by \eqref{eq:a simple fact of epsilon instant regret}, we see that $\lim_{\epsilon \to 0} \norm{x_\epsilon - T(x_\epsilon)}  = 0$

\subsection{Proofs of \cref{pr:dual solutions of EP and VI} }

\begin{proof}[Proof of \cref{pr:dual solutions of EP and VI}]
	Let $x_\* \in X_{\*\*}$. It holds that $\forall x \in \XX$,
	$
	0 \geq \Phi(x,x_\*) = f_{x}(x_\*) - f_x(x) \geq \lr{\nabla f_x(x)}{x_\* -x}
	$,
	which implies $x_\* \in X_{\*}$. The condition for the converse case is obvious.
\end{proof}

\subsection{Proof of \cref{pr:beta-alpha strongly monotone}}
	Because $\nabla f_x$ is $\alpha$-strongly monotone, we can derive
	\begin{align*}
	\lr{\nabla f_x(x) - \nabla f_y(y)}{x-y} &= \lr{\nabla f_x(x) - \nabla f_x(y)}{x-y} + \lr{\nabla f_x(y) - \nabla f_y(y)}{x-y}\\
	&\geq \alpha\norm{x-y}^2 - \norm{\nabla f_x(y) - \nabla f_y(y)}_*\norm{x-y}\\
	&\geq (\alpha-\beta) \norm{x-y}^2
	\end{align*}
	$\forall x,y\in \XX$,
	where the last step is due to $\beta$-\rregularity.

%\subsection{Proof of \cref{lm:Lipschitz continuity of argmin map}}
\subsection{Proof of \cref{pr:contraction condition}}
	The result follows immediately from the following lemma.

	\begin{lemma} \label{lm:Lipschitz continuity of argmin map}
		Suppose $f$ is $(\alpha,\beta)$-\rregular with $\alpha > 0$. Then $F$ in~\cref{th:equivalent problems} is point-valued and $\frac{\beta}{\alpha}$-Lipschitz.
	\end{lemma}
	\begin{proof}
	Let $x^* = F(x)$ and $y^* = F(y)$ for some $x,y \in \XX$.
	By strong convexity, $x^*$ and $y^*$ are unique, and
	$\nabla f_x(\cdot)$ is $\alpha$-strongly monotone; therefore it holds that
	\begin{align*}
	\lr{\nabla f_x(y^*)}{y^* - x^*} &\geq \lr{\nabla f_x(x^*)}{y^* - x^*} +  \alpha\norm{x^* - y^*}^2\\
	& \geq \alpha\norm{x^* - y^*}^2
	\end{align*}
	Since $y^*$ satisfies $\lr{\nabla f_y(y^*)}{x^* - y^*} \geq 0$, the above inequality implies that
	\begin{align*}
	\alpha\norm{x^* - y^*}^2 &\leq \lr{\nabla f_x(y^*)}{y^* - x^*} \\
	&\leq \lr{\nabla f_x(y^*) - \nabla f_y(y^*)}{y^* - x^*} \\
	&\leq \norm{\nabla f_x(y^*) - \nabla f_y(y^*)}_*\norm{y^* - x^*}\\
	&\leq \beta \norm{x-y}\norm{y^* - x^*}
	\end{align*}
	Rearranging the inequality gives the statement.
	\end{proof}

\section{Dual Solution and Strongly Convex Sets} \label{app:strongly convex set}

We show when the strong convexity property of $\XX$ implies the existence of dual solution for VIs.
We first recall the definition of strongly convex sets.
\begin{definition}
	Let $\alpha_\XX \geq 0$.
	A set $\XX$ is called \emph{$\alpha_\XX$-strongly convex} if, for any $x,x' \in \XX$ and  $\lambda \in [0,1]$, it holds that, for all unit vector $v$,
	$
	\lambda x + (1-\lambda) x' +  \frac{\alpha_\XX\lambda (1-\lambda)}{2} \norm{x-x'}^2 v \in \XX
	$.
\end{definition}
When $\alpha_\XX =0$, the definition reduces to usual convexity. Also, we see that this definition implies
$\alpha_\XX \leq \frac{4}{D_\XX}$. In other words, larger sets are less strongly convex. This can also be seen from the lemma below.
\begin{lemma} {\normalfont\citep[Theorem 12]{journee2010generalized}}
	Let $f$ be non-negative, $\alpha$-strongly convex, and $\beta$-smooth on a Euclidean space. Then the set $\{ x | f(x ) \leq r\}$ is $\frac{\alpha}{\sqrt{2r \beta}}$-strongly convex.
\end{lemma}

Here we present the existence result.

\begin{proposition} \label{pr:VI+strongly convex set}
	Let $x^\* \in X^\*$. If $\XX$ is $\alpha_\XX$-strongly convex
	$\forall x \in \XX$, it holds that
	%\begin{align*}
	$
	\lr{F(x^*)}{x-x^*} \geq  \frac{\alpha_\XX}{2} \norm{x-x^*}^2 \norm{F(x^*)}_*
	$. %\end{align*}
	If further $F$ is $L$-Lipschitz, this implies
	$%\begin{align*}
	\lr{F(x)}{x-x^*} \geq  (\frac{\alpha_\XX}{2}\norm{F(x^*)}_* - L) \norm{x-x^*}^2
	$, %\end{align*}
	i.e. when $\alpha_\XX \geq \frac{2L}{\norm{F(x^*)}_*}$, $x^\* \in X_{\*}$.
\end{proposition}

\begin{proof}[Proof of \cref{pr:VI+strongly convex set}]
	Let $g = F(x^\*)$.
	Let $y = \lambda x + (1-\lambda) x^\*$ and $d = -\lambda (1-\lambda) \frac{\alpha_\XX}{2} \norm{x-y}^2 v $, for some $\lambda \in [0,1]$ and some unit vector $v$ to be decided later. By $\alpha_\XX$-strongly convexity of $\XX$, we have $y+d \in \XX$. We can derive
	\begin{align*}
	\lr{g}{x-x^\*} &= \lr{g}{x-y-d} + \lr{g}{y+d -x^\*} \\
	&\geq \lr{g}{x-y} - \lr{g}{d}\\ % &\because x^\*\in\Sol\VI(\XX, F) \\
	&= (1-\lambda)\lr{g}{x-x^\*} - \lr{g}{d}
	\end{align*}
	which implies
	$
	\lr{g}{x-x^\*} \geq \frac{-\lr{g}{d}}{\lambda} = (1-\lambda) \frac{\alpha_\XX}{2} \norm{x-x^\*}^2 \lr{g}{v}
	$.
	Since we are free to choose $\lambda$ and $v$, we can set $\lambda =0$ and $v = \argmax_{v:\norm{v}\leq 1} \lr{g}{v}$, which yields the inequality in the statement.
\end{proof}

\section{Complete Proofs of \cref{sec:monotone ep as col}} \label{app:reduction from EP to COL}

In this section, we describe a general strategy to reduce monotone equilibrium problems (EPs) to continuous online learning (COL) problems.
This reduction can be viewed as refinement and generalization of the classic reduction from convex optimization to adversarial online learning and that from saddle-point problem to two-player adversarial online learning.
In comparison, our reduction
\begin{enumerate}
\item results in a single-player online learning problem, which allows for unified algorithm design
\item   considers potential continuous relationship of the losses between different rounds through the setup of COL, which leads to a predictable online problem amenable to acceleration techniques, such as \citep{rakhlin2013online,juditsky2011solving,cheng2018predictor}.
\item and extends the concept to general convex problems, namely, monotone EPs, which includes of course convex optimization and convex-concave saddle-point problems but also fixed-point problems (FPs), variational inequalities (VIs), etc.
\end{enumerate}

The results here are summarized as \cref{th:monotone EP as COL} and \cref{th:from duel residual to primal residual}.

Here we further suppose $\Phi(x,x)=0$ in the definition of EP. This is not a strong condition. First all the common source problems in introduced below in \cref{sec:ep examples} satisfy this condition. Generally, suppose we have some EP problem with $\Phi'(x,x)>0$ for some $x$. We can define $\Phi(x,x) = \Phi'(x,x') - \Phi'(x,x) $. Then the solution of $\EP(\XX,\Phi)$ are subset of the solution $\EP(\XX,\Phi')$. In other words, allowing $\Phi(x,x)>0$ only makes problem easier. We note that the below reduction and discussion can easily be extended to work directly with EPs with $\Phi(x,x)>0$ by defining instead $f_x(x') = \Phi(x,x') - \Phi(x,x)$, but this will make the presentation less clean.

\subsection{Background: Equilibrium Problems (EPs)}

Let $\XX$ be a compact and convex set in a finite dimensional space.
Let $F: x\times x' \mapsto \Phi(x,x')$ be a bifunction\footnote{We impose convexity and continuity to simplify the setup; similar results hold for subdifferentials and Lipschitz continuity defined based on hemi-continuity.} that is continuous in the first argument, convex in the second argument, and satisfies $\Phi(x,x)=0$.\footnote{As discussed, we concern only EP with $\Phi(x,x)=0$ here} The problem $\EP(\XX,F)$ aims to find $x^\* \in \XX$ such that
\begin{align*}
\Phi(x^\*,x) \geq 0, \qquad  \forall x \in \XX
\end{align*}
Its dual problem $\DEP(\XX,F)$ finds $x_{\*\*} \in \XX$ such that
\begin{align*}
\Phi(x, x_{\*\*}) \leq 0, \qquad  \forall x \in \XX
\end{align*}

Based on the problem's definition, a natural residual (or gap function) of $\EP(\XX,F)$ is %, for $x \in \XX$,
\begin{align*}
r_{ep}(x) \coloneqq - \min_{x'\in\XX}\Phi(x,x')
\end{align*}
which says the degree that the inequality in the EP definition is violated.
A residual for $\DEP(\XX,F)$ can be defined similarly as %, for $x' \in \XX$,
\begin{align*}
r_{dep}(x') \coloneqq \max_{x\in\XX} \Phi(x,x')
\end{align*}
Sometimes EPs are called {maxInf} (or {minSup}) problems~\citep{jofre2014variational}, because
\begin{align*}
x^\* \in \argmin_{x\in\XX} r_{ep}(x) = \arg\max_{x\in\XX} \min_{x'\in\XX}\Phi(x,x')
\end{align*}
In a special case, when $\Phi(\cdot, x)$ is concave. It reduces to a saddle-point problem.

We say a bifunction $F$ is \emph{monotone} if it satisfies
\begin{align*}
\Phi(x,x') + \Phi(x',x)  \leq 0,
\end{align*}
and we say $F$ is skew-symmetric if
\begin{align*}
\Phi(x,x') = - \Phi(x,x'),
\end{align*}
which implies $F$ is monotone. Finally, we say the problem $\EP(\XX,F)$ is monotone, if its bifunction $F$ is monotone.

\subsubsection{Examples} \label{sec:ep examples}

We review some source problems of EPs. Please refer to e.g. ~\citep{jofre2014variational,konnov2000duality} for a more complete survey.

\paragraph{Convex Optimization}
Consider $\min_{x\in\XX} h(x)$ where $h$ is convex. We can simply define
\begin{align*}
\Phi(x,x') = h(x') - h(x)
\end{align*}
which is a skew-symmetric (and therefore monotone) bifunction.

We can also define (following the VI given by its optimality condition)
\begin{align*}
\Phi(x,x') = \lr{\nabla h(x)}{x'-x} .
\end{align*}
We can easily verify that this construction is also monotone
\begin{align*}
\Phi(x,x') + \Phi(x',x) = \lr{\nabla h(x)}{x'-x}  + \lr{\nabla h(x')}{x-x'} = \lr{\nabla h(x)-\nabla h(x')}{x'-x}  \leq 0.
\end{align*}

Suppose $h$ is $\mu$-strongly convex. We can also consider
\begin{align*}
\Phi(x,x') = \lr{\nabla h(x)}{x'-x} + \frac{\mu'}{2}\norm{x'-x}^2
\end{align*}
where $\mu' \leq \mu$. Such $F$ is still monotone:
\begin{align*}
\Phi(x,x') + \Phi(x',x) = \lr{\nabla h(x)-\nabla h(x')}{x'-x}  + \mu'\norm{x'-x}^2  \leq 0.
\end{align*}

\paragraph{Saddle-Point Problem}

Let $\UU$ and $\VV$ to convex and compact sets in a finite dimensional space.
Consider a convex-concave saddle point problem
\begin{align} \label{eq:saddle-point problem}
\min_{u\in\UU} \max_{v\in\VV} \phi(u, v)
\end{align}
in which $\phi$ is continuous, $\phi(\cdot, y)$ is convex, and $\phi(x, \cdot)$ is concave. It is well known that in this case
\begin{align*}
\min_{u\in\UU} \max_{v\in\VV} \phi(u, v) =  \max_{v\in\VV} \min_{u\in\UU} \phi(u, v) \eqqcolon \phi^\*.
\end{align*}

We can define a EP by the bifunction
\begin{align} \label{eq:EP bifunction}
\Phi(x,x') \coloneqq - \phi(u, v') + \phi(u', v).
\end{align}
By definition we have the skew symmetry property, which implies monotonicity.

%We can extend this idea to non-cooperative games, which can be characterized by the EP with the bifunction below
%\begin{align*}
%\Phi(x,x') = \sum_{k=1}^K \psi_k(x'_k , x_{- k}) - \psi_k(x_k , x_{- k})
%\end{align*}

\paragraph{Variational Inequality}
A VI with a vector-valued map $F$ finds $x^\* \in \XX$ such that
\begin{align*}
\lr{F(x^\*)}{x-x^\*} \geq 0,\qquad \forall x \in \XX.
\end{align*}
To turn that into a EP, we can simply define
\begin{align*}
\Phi(x,x') = \lr{F(x)}{x'-x}.
\end{align*}

\paragraph{Mixed Variational Inequality (MVI)}
MVI considers problems that finds $x^\* \in \XX$ such that
\begin{align*}
h(x) - h(x^\*) + \lr{F(x^\*)}{x-x^\*}  \geq 0, \qquad \forall x \in \XX.
\end{align*}
Following the previous idea, we can define its equivalent EP through the bifunction
\begin{align*}
\Phi(x,x') = h(x') - h(x) + \lr{F(x)}{x'-x}
\end{align*}

\subsection{More insights into residuals of primal and dual EPs}

We derive further relationships between primal and dual EPs. These properties will be useful for understanding the reduction introduced in the next section.

\subsubsection{Monotonicity}

By the definition of monotonicity,
$
\Phi(x,x') + \Phi(x',x) \leq 0
$,
we can relate the primal and the dual residuals: for $\hat{x} \in \XX$,
\begin{align*}
r_{dep}(\hat{x}) = \max_{x\in\XX} \Phi(x, \hat{x}) \leq \max_{x \in \XX} - \Phi(\hat{x},x) =   r_{ep}(\hat{x})
\end{align*}
Let $X^\*$ and $X_{\*\*}$ be the solution sets of the EP and DEP, respectively. In other words, for monotone EPs, $X^\* \subseteq X_{\*\*}$.

\subsubsection{Continuity}

When  $\Phi(\cdot, x)$ is continuous, it can be shown that $X^\* \subseteq X_{\*\*}$ \citep{konnov2000duality} (this can be relaxed to hemi-continuity). %, we have $X^\* = X_{\*\*}$ in this case.
Below we relate the primal and the dual residuals in this case.
It implies that the convergence rate of the primal residual is slower than the dual residual.

\begin{proposition}
	Suppose $\Phi(\cdot, x)$ is $L$-Lipschitz continuous for any $x \in \XX$ and $\max_{x,x'\in\XX} \norm{x-x'} \leq D$.
	If $r_{dep} (x) \leq 2 L D$, the $r_{ep}(x) \leq 2 \sqrt{2 LD} \sqrt{r_{dep}(x)}$.

	Suppose in addition $\Phi(x,\cdot)$ is $\mu$-strongly convex with $\mu > 0$.
	If $r_{dep} (x) \leq \frac{L^2}{\mu}$, we can remove the dependency on $D$ and show	$
		r_{ep}(x)  \leq 2.8  (\frac{L^2}{\mu})^{1/3}  r_{dep}(x)^{2/3}
$.
\end{proposition}

\begin{proof}
	Let $y \in \XX$ be arbitrary. Define $z = \tau x + (1-\tau) y$, where $\tau \in [0,1]$. Suppose $x$ is an $\epsilon$-approximate dual solution, i.e.,
	\begin{align*}
	r_{dep}(x) = \max_{x'\in\XX} \Phi(x',x) = \epsilon
	\end{align*}
	By convexity and $\Phi(z,z) = 0$, we can write
	\begin{align*}
	\epsilon \geq \Phi(z,x)
	&= \Phi(z,x) - \Phi(z,z) \\
	&\geq \Phi(z,x) - \tau \Phi(z,x) - (1-\tau) \Phi(z,y)
	= (1-\tau) ( \Phi(z,x)- \Phi(z,y))
	\end{align*}
	%That is, for any $y\in\XX$, we have shown 	$ 	\Phi(z,x)- \Phi(z,y) \leq \frac{\epsilon}{1-\tau} 	$.
	Using this, we can then show
	\begin{align*}
	-\Phi(x,y) &=  -\Phi(x,y) + \Phi(z,y)  + \left( \Phi(z,x)- \Phi(z,y) \right) - \Phi(z,x) + \Phi(x,x)\\
	&\leq  |\Phi(z,y)-\Phi(x,y)| + |\Phi(x,x)-\Phi(z,x)|  +  \Phi(z,x)- \Phi(z,y) \\
	&\leq  2 (1-\tau) L \norm{x-y}   + \Phi(z,x)- \Phi(z,y) &\since{Lipschitz condition}\\
	&\leq  2 (1-\tau) L \norm{x-y}   + \frac{\epsilon}{1-\tau} &\since{The inequality above}\\
	&\leq  2 (1-\tau) L D  + \frac{\epsilon}{1-\tau}
	\end{align*}
	Assume  $\epsilon \leq 2LD$ and let $(1-\tau) = \sqrt{\frac{\epsilon}{2LD}}$, which satisfies $\tau\in[0,1]$. Then
	\begin{align*}
	-\Phi(x,y) \leq 2\sqrt{2LD \epsilon}
	\end{align*}

	When we have $\mu$-strong convexity, we have a tighter bound
	\begin{align*}
	\epsilon \geq \Phi(z,x)
	= \Phi(z,x) - \Phi(z,z)
	&\geq \Phi(z,x) - \tau \Phi(z,x) - (1-\tau) \Phi(z,y) + \frac{\mu \tau(1-\tau)}{2}\norm{x-y}^2\\
	&= (1-\tau)(\Phi(z,x)- \Phi(z,y) ) + \frac{\mu \tau(1-\tau)}{2}\norm{x-y}^2
	\end{align*}
	%That is, for any $y\in\XX$, we have shown 	$ 	\Phi(z,x)- \Phi(z,y) \leq \frac{\epsilon}{1-\tau} 	$.
%	Using t5

	Using this, we can instead show (following similar steps as above)
	\begin{align*}
	-\Phi(x,y)
	&\leq  2 (1-\tau) L \norm{x-y} + \Phi(z,x) - \Phi(z,y)\\
	&\leq  2 (1-\tau) L \norm{x-y}   + \frac{\epsilon}{1-\tau} - \frac{\mu \tau}{2}\norm{x-y}^2\\
	&\leq % \frac{\epsilon}{1-\tau} + \frac{4 L^2 (1-\tau)^2}{2 \mu \tau} \\
	\frac{\epsilon}{1-\tau} + \frac{2L^2  (1-\tau)^2}{ \mu \tau}
	\end{align*}
	where the last inequality is simply $bx - \frac{a}{2}x^2 \leq \frac{b^2}{2a}$ for $a>0$.
	Assume $\epsilon \leq \frac{L^2}{\mu} \eqqcolon \frac{K}{2}$ and let $(1-\tau) = (\frac{\epsilon}{K})^{1/3} \in [0,1]$. We have the following inequality, where the last step uses $\epsilon \leq \frac{K}{2}$.
	\begin{align*}
	-\Phi(x,y) \leq 	\frac{\epsilon}{1-\tau} + \frac{2L^2  (1-\tau)^2}{ \mu \tau}
	= \epsilon^{2/3} K^{1/3} \left( 1 + \frac{1}{1-(\frac{\epsilon}{K})^{1/3}}\right)
	\leq 2.2\epsilon^{2/3} K^{1/3}
 	\end{align*}

\end{proof}

\subsubsection{Equivalence between primal and dual EPs.}
An interesting special case of EP is those with\emph{ skew-symmetric} bifunctions, i.e.
\begin{align*}
\Phi(x,x') = - \Phi(x',x)
\end{align*}
In this case, the EP and the DEP become identical
\begin{align*}
(DEP)\quad \Phi(x, x_{\*\*}) \leq 0  \qquad\Longleftrightarrow\qquad    -\Phi(x_{\*\*},x) \leq 0  \qquad\Longleftrightarrow\qquad
\Phi(x_{\*\*},x) \geq 0  \quad (EP)
\end{align*}
and we have $X^\*  = X_{\*\*}$ and naturally matching residuals
\begin{align*}
r_{dep}(\hat{x})= r_{ep}(\hat{x}).
\end{align*}

Recall from the results of the previous two subsections, generally, when $\Phi(\cdot,x)$ is Lipschitz and $F$ is monotone (but not skew-symmetric), we have $X^\* = X_{\*\*}$ (as known before) but only ($\Phi(x,\cdot)$ is convex)  %(this is less well known).
\begin{align} \label{eq:primal and dual residuals}
r_{dep}(x) \leq  r_{ep}(x) \leq \sqrt{2 LD} \sqrt{r_{dep}(x)}
\end{align}
or ($\Phi(x,\cdot)$ is $\mu$-strongly convex)
\begin{align*}
r_{dep}(x) \leq r_{ep}(x)  \leq 2.8  (\frac{L^2}{\mu})^{1/3}  r_{dep}(x)^{2/3}
\end{align*}

\subsubsection{Relationship with VIs}

We can reduce a EP into a VI problem. We observe that if a point $x^\* \in \XX$ satisfies
\begin{align*}
\Phi(x^\*,x) \geq 0, \qquad  \forall x \in \XX
\end{align*}
if only if
\begin{align*}
\nabla_2 \Phi(x^\*,x^\*)^\t(x - x^\*) \geq 0, \qquad  \forall x \in \XX
\end{align*}
(i.e. $x^\*$ is a global minimum of the function $\Phi(x^\*, \cdot)$), where $\nabla_2$ denotes the partial derivative with respect to the second argument.
 Therefore, $\EP(\XX,\Phi)$ is equivalent to $\VI(\XX,F)$
\begin{align*}
\text{find $x^\* \in \XX$ \quad s.t.} \qquad  \lr{F(x)}{x'-x} \geq 0,\qquad \forall x' \in \XX
\end{align*} if we define $F$ as
\begin{align} \label{eq:vector field based on EP}
F: x \in \XX \mapsto F(x) = \nabla_2 \Phi(x,x)
\end{align}
In a sense, this VI problem is a linearization of the EP problem. In other words, VIs are EPs whose bifunction satisfies that $\Phi(x,\cdot)$ is linear.

By the definition in~\eqref{eq:vector field based on EP}, we can show that
\begin{align*}
r_{dvi}(\hat{x}) \leq r_{dep}(\hat{x}) \qquad \text{and} \qquad  r_{ep}(\hat{x}) \leq r_{vi}(\hat{x})
\end{align*}
And if $\Phi$ is monotone, then $F=\nabla_2 \Phi(x,x)$ is monotone (though the opposite is not true), because
	\begin{align*}
	\lr{F(x)}{x'-x} = \lr{\nabla_2 \Phi(x,x)}{x' -x}
	&\leq \Phi(x,x')  &&\since{Convexity}\\
	&\leq -\Phi(x',x) &&\since{Monotonicity} \\
	&\leq \lr{\nabla_2 \Phi(x',x')}{x' -x} = \lr{F(x')}{x'-x}  &&\since{Convexity}
	\end{align*}
Note the converse is not true, unless $\Phi(x,\cdot)$ is linear.
	%	Above we use the convexity of $\Phi(x,\cdot)$:
	%	\begin{align*}
	%	\Phi(x,x') \geq \Phi(x,x) +  \lr{\nabla_2 \Phi(x,x)}{x'-x}\\
	%	\implies
	%	\lr{\nabla_2 \Phi(x,x)}{x -x'}  \geq - \Phi(x,x') +\Phi(x,x)  = - \Phi(x,x')
	%	\end{align*}
	%	and the definition of monotonicity
	%	\begin{align*}
	%	\lr{\Phi(x)-\Phi(x')}{x-x'} \geq 0
	%	\end{align*}

\subsection{Reduction from Equilibrium Problems to Continuous Online Learning}
\label{app:general reduction}

Now we present the general reduction strategy.
Given a EP $(\XX, \Phi)$, we propose to define a COL problem by identifying
\begin{align*}
f_{x}(x') = \Phi(x,x')
\end{align*}
We can see that this definition is consistent with \cref{th:equivalent problems}: due to $\Phi(x,x)=0$, it satisfies
\begin{align*}
f_{x}(x') - f_{x}(x) = \Phi(x,x') - \Phi(x,x) = \Phi(x,x')
\end{align*}
Therefore, we can say a COL is \emph{normalized} if $f_{x}(x) = 0$. In this case, $f$ and $\Phi$ are interchangeable.

Below we relate the dynamic regret
$
\regret_N^d \coloneqq \sum_{n=1}^{N} f_{x_n} (x_n) - \min_{x\in\XX} f_{x_n} (x)
$
and the static regret
$
\regret_N^s \coloneqq  \sum_{n=1}^{N} f_{x_n} (x_n) - \min_{x\in\XX}  \sum_{n=1}^{N}   f_{x_n} (x)
$
of this problem to the convergence to the EP's solution; note that the above definitions use the fact that in COL $l_n(x) = f_{x_n}(x)$.

\subsubsection{Dynamic Regret and Primal Residual}

We first observe that each instant term in the dynamic regret of this COL problem is exactly the residual function:
\begin{align*}
f_{x_n} (x_n) - \min_{x\in\XX} f_{x_n} (x) = - \min_{x\in\XX} \Phi(x_n,x) = r_{ep}(x_n)
\end{align*}
Therefore, the average dynamic regret describes the rate the gap function converges to zero:
\begin{align*}
\sum_{n=1}^{N} r_{ep}(x_n)   = \sum_{n=1}^{N} f_{x_n} (x_n) - \min_{x\in\XX} f_{x_n} (x)  = \regret_N^d
\end{align*}
Note that the above relationship holds also for weighted dynamic regret.
In general, it means that if the average dynamic regret converges, then the last iterate must converge to the solution set of the EP (since the residual is non-negative.)

\subsubsection{Static Regret and Dual Residual of Monotone EPs}

Next we relate the weighted static regret to the dual residual of the EP.
Let $\{w_n\}$ be such that $w_n > 0$. Let $\hat{x}_N = \frac{1}{x_{1:N}}\sum_{n=1}^{N} w_n x_n$ for some $\{x_n\in \XX \}_{n=1}^N$, where we define $w_{1:N} \coloneqq \sum_{n=1}^{N} w_n$.
We can derive
\begin{align*}
r_{dep}(\hat{x}_N)
&=  \max_{x\in\XX} \Phi(x, \hat{x}_N)  \\
&\leq \max_{x\in\XX} \frac{1}{w_{1:N}}\sum_{n=1}^N w_n \Phi(x, x_n) &&\since{Convexity}\\
&\leq \max_{x\in\XX} \frac{1}{w_{1:N}}\sum_{n=1}^N - w_n \Phi(x_n, x) &&\since{Monotonocity} \\
&= -\min_{x\in\XX} \frac{1}{w_{1:N}}\sum_{n=1}^N  w_n \Phi(x_n, x)   \\
&= \frac{1}{w_{1:N}}\sum_{n=1}^N  w_n \Phi(x_n, x_n)  -\min_{x\in\XX} \frac{1}{w_{1:N}}\sum_{n=1}^N  w_n \Phi(x_n, x) &&\since{$\Phi(x_n,x_n)=0$}\\
&=  \frac{1}{w_{1:N}} \left( \sum_{n=1}^{N} w_n f_n(x_n) - \min_{x\in\XX}  \sum_{n=1}^N w_n f_n(x)  \right)\\
&\eqqcolon \frac{\regret_N^s(w)}{w_{1:N}}
\end{align*}
Note that the inequality
$
r_{dep}(\hat{x}_N)  \leq \frac{\regret_N^s(w)}{w_{1:N}}
$
holds for \emph{any} sequence $\{x_n\}$ and $\{w_n\}$.
Interestingly, by~\eqref{eq:primal and dual residuals}, we see that by the definition of regrets and the property of monotonicity and local Lipschitz continuity, it holds that
\begin{align*}
 \frac{r_{ep}(\hat{x}_N)^2}{2LD} \leq r_{dep}(\hat{x}_N) \leq \frac{\regret_N^s(w)}{w_{1:N}} \leq \frac{\regret_N^d(w)}{w_{1:N}} \eqqcolon \frac{\sum_{n=1}^{N} w_n r_{ep}(x_n) }{w_{1:N}}
\end{align*}
where $L$ is the Lipschitz constant of $\Phi(\cdot, x)$ and $D$ is the size of $\XX$.

\subsection{Summary}

Let us summarize the insights gained from the above discussions.
\begin{enumerate}
\item We can reduce $\EP(\XX,\Phi)$ with monotone $\Phi$ to the COL problem with $l_n(x) = \Phi(x_n,x)$

\item In this COL, the convergence in (weighted) average dynamic regret implies the convergence of the last iterate to the primal solution set. The convergence in (weighted) average static regret implies the convergence of the (weighted) average decision to the dual solution set.

\item Because any dual solution is a primal solution when $\Phi(\cdot, x)$ is continuous, this implies the (weighted) average solution above also converges to the primal solution set. Particularly, if the problem is Lipschitz, we can show $r_{ep} \leq O(\sqrt{r_{dep}})$ and therefore we can also quantify the exact quality of  $\hat{x}_N$ in terms of the primal EP (though it results in a slower rate).

\item When the problem is skew-symmetric (as in the case of common reductions  from optimization and saddle-point problems), we have exactly $r_{ep} = r_{dep}$. This means the average static regret rate directly implies the quality of $\hat{x}_N$ in terms of the primal residual, \emph{without} rate degradation.

%\item Finally, we note that for certain COL problems with strongly convex per-round loss functions, the dynamic regret can be upper bounded by static regret as well~\citep{cheng2018convergence}.
\end{enumerate}

\section{Complete Proofs of \cref{sec:reductions}} \label{app:proofs of reduction}

\subsection{Proof of \cref{th:reduction of dynamic regret}}

The main idea is based on the decomposition that
\begin{equation}\label{eq:decomposition of dynamic regret}
\begin{aligned}
\regret_N^d &= \textstyle \sum_{n=1}^{N} f_{x_n} (x_n)  - f_{x_n} (x^\*)  + \sum_{n=1}^N f_{x_n} (x^\*) - f_{x_n} (x_n^*)
\end{aligned}
\end{equation}
For the first term, $\sum_{n=1}^{N} f_{x_n} (x_n) - f_{x_n} (x^\*) = \regret_N^s(x^\star) \leq \regret_N^s$ and $f_{x_n} (x_n)  - f_{x_n} (x^\*) \leq \lr{\nabla f_{x_n}(x_n)}{x_n - x^\*} \leq G \Delta_n$.
For the second term, we derive
%\begin{small}
	\begin{align*}
	&f_{x_n} (x^\*) - f_{x_n} (x_n^*) \\
	&\leq \lr{\nabla f_{x_n} (x^\*) }{x^\* - x_n^*} - \frac{\alpha}{2}\norm{x^\* - x_n^*}^2 \\ %& \text{$\because \alpha$-strong convexity} \\
	&\leq   \lr{\nabla f_{x_n} (x^\*)- \nabla f_{x^\*}(x^\*) }{x^\* - x_n^*} - \frac{\alpha}{2}\norm{x^\* - x_n^*}^2 \\ % & \text{$\because x^\* \in \Sol( \VI(\XX, \nabla f)$)} \\
	&\leq  \norm{\nabla f_{x_n} (x^\*)- \nabla_{x^\*}f(x^\*) }_*\norm{x^\* - x_n^*} - \frac{\alpha}{2}\norm{x^\* - x_n^*}^2 \\ % & \text{$\because x^\* \in \Sol( \VI(\XX, \nabla f)$)} \\
	&\leq  \beta\norm{x_n -x^\*}\norm{x^\* - x_n^*} - \frac{\alpha}{2}\norm{x^\* - x_n^*}^2 \\ % & \text{$\because x^\* \in \Sol( \VI(\XX, \nabla f)$)} \\
	&\leq  \min\{ \beta D_\XX \norm{x_n -x^\*},   \frac{\beta^2}{2\alpha}\norm{x_n - x^\*}^2 \}
	\end{align*}
%\end{small}%
in which the second inequality is due to that $x^\* \in X^\*$ and the fourth inequality is due to $\beta$-\rregularity. Combining the two terms gives the upper bound.
For the lower bound, we notice that when $x_\* \in X_\*$, we have $ f_{x_n} (x_n)  - f_{x_n} (x_\*) \geq 0$. Since by~\cref{pr:primal and dual solutions} $x_\* \in X^\*$ is also true, we can use~\eqref{eq:decomposition of dynamic regret} and the fact that $f_{x_n} (x_\*) - f_{x_n} (x_n^*) \geq \frac{\alpha}{2} \norm{x_\* - x_n^*}^2$ to derive  the lower bound.

\subsection{Proof of \cref{cr:full reduction to static regret}}

%\begin{proof}[Proof of \cref{cr:full reduction to static regret}]
	By~\cref{pr:beta-alpha strongly monotone}, $\nabla f$ is $(\alpha-\beta)$-strongly monotone, implying
	$
	\lr{\nabla f_{x_n}(x_n)}{x_n - x^\*} \geq (\alpha -\beta) \Delta_n^2
	$, where we recall that $\Delta_n = \norm{x_n - x^\*}$ and $x^\* \in X^\*$.
	 Because $\sum_{n=1}^N \lr{\nabla f_{x_n}(x_n)}{x_n - x^\*} =  \widetilde{\regret_N^s}(x^\star) \leq \widetilde{\regret_N^s}$, we have by~\cref{th:reduction of dynamic regret} the inequality in the statement.
%\end{proof}

\subsection{Proof of~\cref{pr:alpha equals beta full information}}

%\begin{proof}[Proof of Proposition~\ref{pr:alpha equals beta full information}]
	In this case, by~\cref{pr:contraction condition}, $T$ is non-expansive. We know that, e.g., Mann iteration~\citep{mann1953mean}, i.e., for $\eta_n \in (0,1)$ we set
	\begin{align} \label{eq:Mann iteration}
	x_{n+1} = \eta_n x_n + (1-\eta_n) x_n^*,
	\end{align}
	converges to some $x^\* \in X^\*$; in view of~\eqref{eq:Mann iteration}, the greedy is update is equivalent to Mann iteration with $\eta_n =1$.
	%	\cheng{Need to double check this }
	As Mann iteration converges in general Hilbert space, by~\cref{th:equivalent problems}, it has sublinear dynamic regret with some constant that is polynomial in $d$.
%\end{proof}

\subsection{Proof of~\cref{pr:mirror descent dynamic regret}}

%\begin{proof}[Proof of~\cref{pr:mirror descent dynamic regret}]

We first establish a simple lemma related to the smoothness of $\nabla f_x(x)$ and then a result on the convergence of the Bregman divergence $B_R(x_n \| x^\star)$. The purpose of the second lemma is to establish essentially a contraction showing that the distance between the equilibrium point $x^\star$ and $x_n$ strictly decreases.

\begin{lemma}\label{lm:gamma plus beta smooth}
If, $\forall x\in \XX$, $\nabla f_\cdot (x)$ is $\beta$-Lipschitz continuous and $f_x(\cdot)$ is $\gamma$-smooth, then, for any $x,y \in \XX$,
\begin{align*}
\| \nabla f_x(x) - \nabla f_y(y) \|_* \leq (\gamma + \beta) \|x - y\|.
\end{align*}
\end{lemma}
\begin{proof}
For any $x, y \in \XX$, it holds that
\begin{align*}
\| \nabla f_x(x) - \nabla f_y(y) \|_* & \leq \| \nabla f_x(x) - \nabla f_y(x) + \nabla f_y(x) - \nabla f_y(y) \|_* \\
& \leq \| \nabla f_x(x) - \nabla f_y(x) \|_* + \|\nabla f_y(x) - \nabla f_y(y) \|_* \\
& \leq \beta \| x - y\| + \gamma \|x - y\|.
\end{align*}
The last inequality uses $\beta$-regularity and $\gamma$-smoothness of $\nabla f_\cdot(x)$ and $f_y(\cdot)$, respectively.
\end{proof}

%<<<<<<< HEAD

%=======
%\begin{proof}[Proof of Proposition~\ref{pr:alpha equals beta full information}]
%	In this case, by~\cref{pr:contraction condition}, $T$ is non-expansive. We know that, e.g., Mann iteration~\citep{mann1953mean}, i.e., for $\eta_n \in (0,1)$ we set
%	\begin{align} \label{eq:Mann iteration}
%	x_{n+1} = \eta_n x_n + (1-\eta_n) x_n^*,
%	\end{align}
%	converges to some $x^\* \in X^\*$; in view of~\eqref{eq:Mann iteration}, the greedy is update is equivalent to Mann iteration with $\eta_n =1$.
%%	\cheng{Need to double check this }
%	As Mann iteration converges in general Hilbert space, by~\cref{th:equivalent problems}, it has sublinear dynamic regret with some constant that is polynomial in $d$.
%\end{proof}
%>>>>>>> 45de895e2cdf1d7d3272c3192a80fc1e3472e536

\begin{lemma}\label{lm:mirror descent contraction}
	If $f$ is $(\alpha,\beta)$-regular, $f_x(\cdot)$ is $\gamma$-smooth for all $x \in \XX$, and $R$ is $1$-strongly convex and $L$-smooth, then for the online mirror descent algorithm it holds that
	\begin{align*}
	& B_R(x^\star \| x_{n}) \leq \left(1 - 2\eta (\alpha - \beta)L^{-1} + \eta^2 (\gamma + \beta)^2 \right)^{n-1} B_R(x^\star \| x_{1}).
	\end{align*}
\end{lemma}

\begin{proof}
By the mirror descent update rule in (\ref{eq:mirror descent update}), $\< \eta \nabla f_{x_n}(x_n) + \nabla R (x_{n+1}) - \nabla R (x_n), x^\* - x_{n+1}\> \geq 0$. Since $x^\star \in X_{\*}$, $\<\eta \nabla f_{x^\*}(x^\*), x_{n+1} - x^\star \> \geq 0$. Combining these inequalities yields $\eta\< \nabla f_{x_n}(x_n) - \nabla f_{x^\*} (x^\star), x_{n+1} - x^\star\> \leq \< \nabla R(x_{n+1}) - \nabla R(x_n), x^\star - x_{n+1}\>$.
Then by the three-point equality of the Bregman divergence, we have
\begin{align*}
B_R(x^\star \| x_{n+1})
& \leq B_R(x^\star \| x_{n}) - B_R(x_{n+1} \| x_{n})  - \eta \< \nabla f_{x_n}(x_n) - \nabla f_{x^\*} (x^\star), x_{n+1} - x^\star\>.
\end{align*}
Because of the $(\alpha - \beta)$-strong monotonicity of $\nabla f_x(x)$, the above inequality implies
\begin{align*}
B_R(x^\star \| x_{n+1}) & \leq B_R(x^\star \| x_{n}) - B_R(x_{n+1} \| x_{n}) - \eta \< \nabla f_{x_n}(x_n) - \nabla f_{x^\*} (x^\star), x_{n+1} - x_n\> \\
& \quad - \eta \< \nabla f_{x_n}(x_n) - \nabla f_{x^\*} (x^\star), x_{n} - x^\star\> \\
& \leq B_R(x^\star \| x_{n}) - B_R(x_{n+1} \| x_{n}) - \eta \< \nabla f_{x_n}(x_n) - \nabla f_{x^\*} (x^\star), x_{n+1} - x_n\> - \eta(\alpha - \beta) \|x^\star - x_n\|^2\\
& \leq B_R(x^\star \| x_{n}) + \frac{\eta^2(\gamma + \beta)^2}{2}\|x^\star - x_n\|^2 - \eta (\alpha - \beta) \|x^\star - x_n\|^2 \\
& \leq \left( 1 + \eta^2(\gamma + \beta)^2 - 2\eta (\alpha - \beta) L^{-1} \right) B_R(x^\star\|x_n).
\end{align*}
The third inequality results  from the Cauchy-Scwharz inequality followed by maximizing over $\|x_{n+1} - x_n\|$ and then applying Lemma \ref{lm:gamma plus beta smooth}. The last inequality uses the fact that $R$ is $1$-strongly convex and $L$-smooth.
%\cheng{
%We don't need the norm be contractive. Say we can show $B_R(x^\* \| x_N ) = o(N)$; it already implies $\norm{x^\* - x_N}^2 \leq O(B_R(x^\* \| x_N )) \leq o(N)$ which is sufficient to show the result.
%}
%\lee{Here the Bregman divergences form a ``contraction" but the norm squared terms are not necessarily converging to zero. Can cleverly choose $\eta < \frac{(\alpha - \beta) + \sqrt{(\alpha - \beta)^2 + (\gamma + \beta)^2 L - (\gamma + \beta)^2 L^2}}{L(\gamma + \beta)^2}$, but I think this results in an additional constraint, not just $\alpha > \beta$.}
%
%\lee{
%	Old version requiring that $R$ be $1$-smooth:
%\begin{align*}
%\Delta_{n+1}^2 & \leq \Delta_n^2 - \|x_n - x_{n+1}\|^2 - 2 \eta\< \nabla f_{x_n}(x_n) - \nabla f_{x^\*} (x^\star), x_{n+1} - x^\star \> \\
%& = \Delta_n^2 - \|x_n - x_{n+1}\|^2 - 2 \eta\< \nabla f_{x_n}(x_n) - \nabla f_{x^\*} (x^\star), x_{n} - x^\star \> - 2 \eta\< \nabla f_{x_n}(x_n) - \nabla f_{x^\*} (x^\star), x_{n+1} - x_n \> \\
%&\leq (1 - 2\eta (\alpha - \beta)) \Delta_n^2 - \|x_n - x_{n+1}\|^2 + 2\eta \| \nabla f_{x_n}(x_n) - \nabla f_{x^\*} (x^\star)\|_* \| x_{n+1} - x_n \| \\
%&\leq (1 - 2\eta (\alpha - \beta) + \eta^2 (\gamma + \beta)^2) \Delta_n^2,
%\end{align*}
%}
%where the second inequality uses $(\alpha - \beta)$-strong monotonicity of $\nabla f_x$ and the Cauchy-Schwarz inequality. The third inequality maximizes over $\| x_{n+1} - x_n \|$ and then uses Lemma \ref{lm:gamma plus beta smooth}.
\end{proof}

If $\alpha > \beta$ and $\eta$ is chosen such that $\eta < \frac{2(\alpha - \beta)}{L(\gamma + \beta)^2}$, we can see that the online mirror descent algorithm guarantees linear convergence of $B_R(x^\star \| x_n)$ to zero with rate $(1 - 2\eta (\alpha - \beta)L^{-1} + \eta^2 (\gamma + \beta)^2) \in (0, 1)$. By strong convexity, we have,
\begin{align*}
	\Delta_n = \|x^\star - x_{n}\| & \leq \sqrt{2 B_R(x^\star \| x_{n})} \\
	& \leq \sqrt 2 \left( 1 + \eta^2(\gamma + \beta)^2 - 2\eta (\alpha - \beta) L^{-1} \right)^{\frac{n - 1}{2}} B_R(x^\star \| x_0)^{1/2}.
\end{align*}

The proposition follows immediately from combining this result and Theorem \ref{th:reduction of dynamic regret}.
%\end{proof}

\subsection{Proof of~\cref{pr:stochastic mirror descent}}
Recall that $g_n = \nabla l_n(x_n) + \epsilon_n + \xi_n$.
 As discussed previously, we assume there exist constants $0 \leq \sigma, \kappa < \infty$ such that $\E \left[\| \epsilon_n\|_*^2\right] \leq \sigma^2$ and $\|\xi_n\|_*^2 \leq \kappa^2$ for all $n$.
 The  mirror descent update rule is given by
\begin{align}\label{eq:stochastic mirror descent update}
x_{n+1} = \argmin_{x \in \XX}\<\eta_n g_n, x \> + B_R(x\|x_n).
\end{align}

%\begin{proof}[Proof of~\cref{pr:stochastic mirror descent}]
We use~\cref{cr:full reduction to static regret} along with known results for the static regret to bound the dynamic regret in the stochastic case.
The main idea of the proof is to show the result for the linearized losses. By convexity, this can be used to bound both terms in~\cref{cr:full reduction to static regret}.

Let $u$ be any fixed vector in $\XX$, chosen independent of the learner's decisions $x_1, \ldots, x_n$.
%For convenience, denote $\phi_n = \phi_n(x_n)$.
The first-order condition for optimality of \eqref{eq:stochastic mirror descent update} yields $\< \eta_n g_n, x_{n+1} - u\> \leq \< u - x_{n+1} , \nabla R(x_{n+1}) - \nabla R(x_n)\>$.
We use this condition to bound the linearized losses as in the proof of~\cref{pr:mirror descent dynamic regret}. We can bound the linearized losses by the magnitude of the stochastic gradients and Bregman divergences between $u$ and the learner's decisions:
\begin{align*}
\< g_n, x_{n} - u\> & \leq \frac{1}{\eta_n} \< u - x_{n+1} , \nabla R(x_{n+1}) - \nabla R(x_n)\> + \< g_n, x_{n} - x_{n+1}\>\\
& = \frac{1}{\eta_n} B_R(u\|x_n) - \frac{1}{\eta_n}B_R(u\|x_{n+1}) - \frac{1}{\eta_n}B_R(x_{n+1}\| x_n)  + \< g_n, x_{n} - x_{n+1}\> \\
& \leq \frac{1}{\eta_n} B_R(u\|x_n) - \frac{1}{\eta_n}B_R(u\|x_{n+1}) - \frac{1}{2\eta_n}\|x_n - x_{n+1}\|^2  + \|g_n\|_* \|x_{n} - x_{n+1}\|\\
& \leq \frac{1}{\eta_n} B_R(u\|x_n) -\frac{1}{\eta_n} B_R(u\|x_{n+1}) + \frac{\eta_n}{2} \|g_n\|^2_*.
\end{align*}
The first inequality follows from adding $\<g_n, x_n - x_{n+1}\>$ to both sides of the inequality from the first-order condition for optimality. The equality uses the three-point equality of the Bregman divergence.
The second inequality follows from the Cauchy-Schwarz inequality and the fact that $\frac{1}{2}\|x_n - x_{n+1}\|^2 \leq B_R(x_{n+1} \| x_n)$ due to the $1$-strong convexity of $R$.
%The second inequality follows the Cauchy-Schwarz inequality and the $1$-strong convexity of $R$ and then maximizing over $\|x_n - x_{n+1}\|$.
The last inequality maximizes over $\|x_n - x_{n+1}\|$.

Define $\RR = \sup_{w_1,w_2 \in \XX}B_R(w_1 \| w_2)$, which is bounded. Note that $\E\left[ \| g_n\|_*^2 \right] \leq 3 (G^2 + \sigma^2 + \kappa^2)$. Therefore, summing from $n = 1$ to $N$, it holds for any $u\in\XX$ selected before learning,
\begin{align*}
\E \left[ \sum_{n = 1}^N \< g_n, x_{n} - u \> \right] & \leq \E \left[ \sum_{n = 1 }^N
  \left(\frac{1}{\eta_n} - \frac{1}{\eta_{n-1}} \right) \RR + \frac{3}{2}(G^2 + \sigma^2 + \kappa^2)\eta_n \right]
\end{align*}
After rearrangement, we have
\begin{align*}
\E \left[ \sum_{n = 1}^N \< \nabla l_n(x_n) + \epsilon_n, x_{n} - u \> \right] & \leq \E \left[ \sum_{n = 1 }^N   \left(\frac{1}{\eta_n} - \frac{1}{\eta_{n-1}} \right) \RR + \frac{3}{2}(G^2 + \sigma^2 + \kappa^2)\eta_n +  D_\XX \|\xi_n\|_* \right].
\end{align*}
Choosing $\eta_n = \frac{1}{\sqrt n}$, $\eta_n = \eta_1$, and $u = x^\star$ (because $x^\star$ is fixed for a fixed $f$ selected before learning) yields $\E \left[ \sum_{n = 1}^N \< \nabla l_n(x_n) + \epsilon_n , x_{n} - x^\star \> \right] = O(\sqrt N + \Xi)$.
Because of the law of total expectation and that $x_n$ does not depend on $\epsilon_n$, we have $\E [\widetilde{\regret_N^s}(x^\star)] =\E \left[ \sum_{n = 1}^N \< \nabla l_n(x_n) + \epsilon_n, x_{n} - x^\star \> \right]$.
Further, by convexity, it follows $\E [\regret_N^s(x^\star)] \leq \E [\widetilde{\regret_N^s}(x^\star)]$. Then, we may apply~\cref{cr:full reduction to static regret} to obtain the result. Note that there is no requirement that $R$ is smooth.

\section{Complete Proofs of \cref{sec:extensions}} \label{app:proofs of extensions}

\subsection{Proof of~\cref{pr:generalized contraction property}}

%\begin{proof}[Proof of \cref{pr:generalized contraction property}]
	Because
	$\nabla l_n(\cdot)$ is $\alpha$-strongly monotone, it holds
	\begin{align*}
	\lr{\nabla l_n(x_{n-1}^*)}{x_{n-1}^* - x_n^*} & \geq \alpha\norm{x_{n-1}^* - x_n^*}^2
	\end{align*}
	Since $y^*$ satisfies $\lr{\nabla l_{n-1}(x_{n-1}^*)}{x_n^* - x_{n-1}^*} \geq 0$, the above inequality implies that
	\begin{align*}
	\alpha\norm{x_n^* - x_{n-1}^*}^2
	&\leq \lr{\nabla l_n(x_{n-1}^*) - \nabla l_{n-1}(x_{n-1}^*)}{x_{n-1}^* - x_n^*} \\
	&\leq (\beta \norm{x_n-x_{n-1}} + a_n)\norm{x_{n-1}^* - x_n^*}
	\end{align*}
	Rearranging the inequality gives the statement.
%\end{proof}

\subsection{Proof of~\cref{th:predictable problem with }}
%\begin{proof}[Proof of~\cref{th:predictable problem with }]
	For convenience, define $\lambda := \frac{\beta}{\alpha}$. Recall that, by the mirror descent update rule, the first-order conditions for optimality of both $x_{x+1}$ and $x_n^*$ yield, for all $x \in \XX$,
	\begin{align*}
	\< \eta \nabla l_n(x_n), x - x_{n+1}\> & \geq \<\nabla R(x_{n}) - \nabla R(x_{n+1}), x - x_{n+1}\> \\
	\< \nabla l_n(x_n^*), x - x_n^* \> & \geq 0.
	\end{align*}
	The proof requires many intermediate steps, which we arrange in a series of lemmas that typically follow from each other in order. Ultimately, we aim to achieve a result that resembles a contraction as done in~\cref{pr:mirror descent dynamic regret} but with additional terms due to the adversarial component of the predictable problem. We begin with general bounds on the Bregman divergence beteween the learner's decisions and the optimal decisions.

	\begin{lemma}\label{lm:predictable current round bound}
		At round $n$, for an $(\alpha, \beta)$-predictable problem under the mirror descent algorithm, if $l_n$ is $\gamma$-smooth and $R$ is $1$-strongly convex and $L$-smooth, then it holds that
	\begin{align*}
		B_R(x_{n+1}^*\| x_{n+1}) & \leq B_R( x_{n+1}^*\| x_n^*) + B_R(x_n^*\| x_{n+1})  \\
		& \quad + \lambda \| x_{n+1} - x_n\| \| \nabla R(x_n^*) - \nabla R(x_{n+1})\|_* + \frac{a_n}{\alpha} \| \nabla R(x_n^*) - \nabla R(x_{n+1})\|_*
	\end{align*}
	and, in the next round,
	\begin{align*}
	B_R(x_n^*\| x_{n+1}) & \leq B_R(x_n^*\| x_n) - B_R(x_{n+1}\| x_n) - \alpha \eta \|x_n - x_n^*\|^2 + \eta \gamma \|x_n - x_n^*\| \|x_{n+1} - x_n\|.
	\end{align*}
	\end{lemma}
	\begin{proof}
		The first result uses the basic three-point equality of the Bregman divergence followed by the Cauchy-Schwarz inequality and \cref{pr:generalized contraction property}. Note that this first part of the lemma does not require that $x_n$ is generated from a mirror descent algorithm:
		\begin{align*}
		B_R(x_{n+1}^*\| x_{n+1}) & = B_R( x_{n+1}^*\| x_n^*) + B_R(x_n^*\| x_{n+1}) + \< x_{n+1}^* - x_n^*, \nabla R(x_n^*) - \nabla R(x_{n+1})\> \\
		& \leq B_R( x_{n+1}^*\| x_n^*) + B_R(x_n^*\| x_{n+1}) + \| x_{n+1}^* - x_n^*\| \| \nabla R(x_n^*) - \nabla R(x_{n+1})\|_* \\
		& \leq B_R( x_{n+1}^*\| x_n^*) + B_R(x_n^*\| x_{n+1})  \\
		& \quad + \lambda \| x_{n+1} - x_n\| \| \nabla R(x_n^*) - \nabla R(x_{n+1})\|_* + \frac{a_n}{\alpha} \| \nabla R(x_n^*) - \nabla R(x_{n+1})\|_*.
		\end{align*}

		For the second part of the lemma, we require using the first-order conditions of optimality of both $x_{n+1}$ for the mirror descent update and $x_n^*$ for $l_n$:
		\begin{align*}
		B_R(x_n^*\| x_{n+1}) & = B_R(x_n^*\| x_n) - B_R(x_{n+1}\| x_n) + \< x_n^* - x_{n+1}, \nabla R(x_n)  - \nabla R(x_{n+1})\> \\
		& \leq  B_R(x_n^*\| x_n) - B(x_{n+1}\| x_n) + \eta \< \nabla l_n(x_n^*) - \nabla l_n(x_n), x_n - x_n^*\> \\
		& \quad  +  \eta \< \nabla l_n(x_n^*) - \nabla l_n(x_n), x_{n+1} - x_n\> \\
		& \leq B_R(x_n^*\| x_n) - B_R(x_{n+1}\| x_n) - \alpha \eta \|x_n - x_n^*\|^2 + \eta \gamma \|x_n - x_n^*\| \|x_{n+1} - x_n\|.
	\end{align*}
	The first line again applies the three-point equality of the Bregman divergence.
	The second line combines both first-order optimality conditions to bound the inner product.
	The last inequality uses the strong convexity of $l_n$ to bound $\eta \< \nabla l_n(x_n^*) - \nabla l_n(x_n), x_n - x_n^*\> \leq -\alpha \eta \|x_n - x_n^*\|^2$ and the Cauchy-Schwarz inequality along with the smoothness of $l_n$ to bound the other inner product.
	\end{proof}

	The second result also leads to a natural corollary that will be useful later in the full proof.

	\begin{corollary}\label{cor:predictable current round}
		Under the same conditions as \cref{lm:predictable current round bound}, it holds that
		\begin{align*}
		B_R(x_n^*\| x_{n+1}) & = \left( 1 - 2\alpha \eta L^{-1} + \eta^2 \gamma^2 \right)B_R(x_n^*\| x_n).
		\end{align*}
	\end{corollary}

	\begin{proof}
		We start with the first inequality of \cref{lm:predictable current round bound} and then maximize over $\|x_{n+1} - x_n\|^2$. Finally, we applying the strong convexity and smoothness of $R$ to achieve the result:
		\begin{align*}
		B_R(x_n^*\| x_{n+1}) & \leq B_R(x_n^*\| x_n) - B_R(x_{n+1}\| x_n) - \alpha \eta \|x_n - x_n^*\|^2 + \eta \gamma \|x_n - x_n^*\| \|x_{n+1} - x_n\| \\
		& \leq (1 - 2\alpha \eta L^{-1})B_R(x_n^*\| x_n) - \frac{1}{2} \| x_{n+1} - x_n\|^2 + \eta \gamma \|x_n - x_n^*\| \|x_{n+1} - x_n\| \\
		& \leq  (1 - 2\alpha \eta L^{-1})B_R(x_n^*\| x_n) + \eta^2 \gamma^2 B_R(x_n^*\| x_n) = \left( 1 - 2\alpha \eta L^{-1} + \eta^2 \gamma^2 \right)B_R(x_n^*\| x_n).\qedhere
		\end{align*}
	\end{proof}

%	By the three-point equality, we have
%	\begin{align*}
%	B_R(x_{n+1}^\star\| x_{n+1}) & = B_R( x_{n+1}^*\| x_n^*) + B_R(x_n^*\| x_{n+1}) + \< x_{n+1}^* - x_n^*, \nabla R(x_n^*) - \nabla R(x_{n+1})\> \\
%	& \leq B_R( x_{n+1}^*\| x_n^*) + B_R(x_n^*\| x_{n+1})  \\
%	& \quad + \lambda \| x_{n+1} - x_n\| \| \nabla R(x_n^*) - \nabla R(x_{n+1})\|_* + \frac{a_n}{\alpha} \| \nabla R(x_n^*) - \nabla R(x_{n+1})\|_*.
%	\end{align*}
%	Similarly,
%	\begin{align*}
%	B_R(x_n^*\| x_{n+1}) & = B_R(x_n^*\| x_n) - B_R(x_{n+1}\| x_n) + \< x_n^* - x_{n+1}, \nabla R(x_n)  - \nabla R(x_{n+1})\> \\
%	& \leq  B_R(x_n^*\| x_n) - B(x_{n+1}\| x_n) + \eta \< \nabla \ell_n(x_n^*) - \nabla \ell_n(x_n), x_n - x_n^*\> \\
%	& \quad  +  \eta \< \nabla \ell_n(x_n^*) - \nabla \ell_n(x_n), x_{n+1} - x_n\> \\
%	& \leq B_R(x_n^*\| x_n) - B_R(x_{n+1}\| x_n) - \alpha \eta \|x_n - x_n^*\|^2 + \eta \gamma \|x_n - x_n^*\| \|x_{n+1} - x_n\|,
%	\end{align*}
%	where the second inequality uses the optimality of $x_n^*$ and the optimality of the update of mirror descent.
%	The last inequality follows from the strong convexity of $\ell_n$.
	We can combine both results of ~\cref{lm:predictable current round bound} in order to show
	\begin{align*}
	B_R(x_{n+1}^*\| x_{n+1}) & \leq B_R( x_{n+1}^*\| x_n^*) + \lambda \| x_{n+1} - x_n\| \| \nabla R(x_n^*) - \nabla R(x_{n+1})\|_* + \frac{a_n}{\alpha} \| \nabla R(x_n^*) - \nabla R(x_{n+1})\|_* \\
	& \quad + B_R(x_n^*\| x_n) - B(x_{n+1}\| x_n) - \alpha \eta \|x_n - x_n^*\|^2 + \eta \gamma \|x_n - x_n^*\| \|x_{n+1} - x_n\|.
	\end{align*}
	Some of the terms in the above inequality can be grouped and bounded above.
	By $L$-smoothness of $R$, we have $B_R(x_{n+1}^*\| x_n^*) \leq \frac{L}{2}\|x_{n+1}^* - x_n^*\|^2 \leq \frac{L}{2}\left(\lambda \|x_n - x_{n+1}\| + \frac{a_n}{\alpha}\right)^2 = \frac{L}{2}\left(\lambda^2 \|x_n - x_{n+1}\|^2 + \frac{a_n^2}{\alpha^2} + \frac{2a_n\lambda}{\alpha} \|x_n - x_{n+1}\|\right)
	%\frac{L\lambda^2}{2}\|x_{n+1} - x_n\|^2
	$.
	Because, $R$ is $1$-strongly convex, $L \geq 1$; therefore, the previous inequality can be bounded from above using $L^2$ instead of $L$.
	While this artificially worsens the bound, it will be useful for simplifying the conditions sufficient for sublinear dynamic regret.
	$1$-strong convexity of $R$ also gives us $-B_R(x_{n+1}, x_n) \leq -\frac{1}{2}\|x_{n+1} - x_n\|^2$. Applying these upper bounds and then aggregating terms yields
	\begin{align*}
	B_R(x_{n+1}^*\| x_{n+1}) & \leq -\frac{(1 - L^2\lambda^2)}{2}\|x_n - x_{n+1}\|^2 + \left( \lambda \| \nabla R(x_n^*) - \nabla R(x_{n+1}) \|_*  + \eta \gamma \|x_n - x_n^*\|\right) \| x_n - x_{n+1}\| \\
	& \quad + B_R(x_n^*\| x_n) - \alpha \eta \|x_n - x_n^*\|^2 + \frac{a_n}{\alpha} \| \nabla R(x_n^*) - \nabla R(x_{n+1})\|_* + \frac{a_n^2 L}{2\alpha^2} + \frac{a_nL\lambda}{\alpha}\|x_n - x_{n+1}\| \\
	& \leq -\frac{(1 - L^2\lambda^2)}{2}\|x_n - x_{n+1}\|^2 + \left( \lambda \| \nabla R(x_n^*) - \nabla R(x_{n+1}) \|_*  + \eta \gamma \|x_n - x_n^*\|\right) \| x_n - x_{n+1}\| \\
	& \quad + B_R(x_n^*\| x_n) - \alpha \eta \|x_n - x_n^*\|^2 + \frac{a_nL}{\alpha}D_\XX + \frac{a_n^2 L}{2\alpha^2} + \frac{a_nL\lambda}{\alpha}D_\XX \\
	& \leq \frac{\lambda^2 \|\nabla R(x_n^*) - \nabla R(x_{n+1})\|_*^2 + \eta^2 \gamma^2 \|x_n - x_n^*\|^2 }{1 - L^2\lambda^2}+ B_R(x_n^*\| x_n) - \alpha \eta \|x_n - x_n^*\|^2  + \zeta_n \\
	& \leq \frac{\lambda^2L^2 \|x_n^* - x_{n+1}\|^2 + \eta^2 \gamma^2 \|x_n - x_n^*\|^2 }{1 - L^2\lambda^2}+ B_R(x_n^*\| x_n) - \alpha \eta \|x_n - x_n^*\|^2  + \zeta_n  \\
	& \leq \frac{2\lambda^2L^2 B_R(x_n^*\| x_{n+1}) + 2\eta^2 \gamma^2 B_R(x^*_n\| x_n) }{1 - L^2\lambda^2}+ B_R(x_n^*\| x_n) - \alpha \eta \|x_n - x_n^*\|^2  + \zeta_n ,
	\end{align*}
	where $\zeta_n = \frac{a_n L D_\XX}{\alpha} \left(1 + \lambda \right) + \frac{a_n^2 L}{2\alpha^2}$. The third inequality follows from maximizing over $\|x_n - x_{n+1}\|$ and then applying $(a + b)^2 \leq 2a^2 + 2b^2$ for any $a, b \in \R$.
	For this operation, we require that $L^2\lambda^2 < 1$.
	The fourth inequality uses $L$-smoothness of $R$.
	The last inequality uses the fact that $R$ is $1$-strongly convex to bound the squared normed differences by the Bregman divergence.

%	From an early derivation in this proof, we know
%	\begin{align*}
%	B_R(x_n^*\| x_{n+1}) & \leq B_R(x_n^*\| x_n) - B_R(x_{n+1}\| x_n) - \alpha \eta \|x_n - x_n^*\|^2 + \eta \gamma \|x_n - x_n^*\| \|x_{n+1} - x_n\| \\
%	& \leq (1 - 2\alpha \eta L^{-1})B_R(x_n^*\| x_n) - \frac{1}{2} \| x_{n+1} - x_n\|^2 + \eta \gamma \|x_n - x_n^*\| \|x_{n+1} - x_n\| \\
%	& \leq  (1 - 2\alpha \eta L^{-1})B_R(x_n^*\| x_n) + \eta^2 \gamma^2 B_R(x_n^*\| x_n) = \left( 1 - 2\alpha \eta L^{-1} + \eta^2 \gamma^2 \right)B_R(x_n^*\| x_n).
%	\end{align*}

	We then use \cref{cor:predictable current round} to bound this result on $B_R(x_{n+1}^*\| x_{n+1})$ in terms of only $B_R(x_{n}^*\| x_{n})$ and the appropriate constants:
	\begin{align*}
	B_R(x_{n+1}^*\| x_{n+1}) & \leq  \frac{2L^2\lambda^2 B_R(x_n^*\| x_{n+1}) + 2\eta^2 \gamma^2 B_R(x^*_n\| x_n) }{1 - L^2\lambda^2}+ B_R(x_n^*\| x_n) - \alpha \eta \|x_n - x_n^*\|^2 + \zeta_n  \\
	&\leq \frac{2L^2\lambda^2}{1 - L^2\lambda^2}\left( 1 - 2\alpha \eta L^{-1} + \eta^2 \gamma^2 \right)B_R(x_n^*\| x_n) + \frac{2\eta^2 \gamma^2}{1 - L^2\lambda^2} B_R(x_n^*\| x_n)  \\
	& \quad  + B_R(x_n^*\| x_n) - 2\alpha \eta L^{-1} B_R(x_n^*\| x_n) + \zeta_n  \\
	& = \left(1 - 2 \alpha \eta L^{-1} + \frac{2\eta^2\gamma^2}{1 - L^2\lambda^2} +  \frac{2L^2\lambda^2}{1 - L^2\lambda^2} - \frac{4L\lambda^2\alpha \eta}{1 - L^2\lambda^2} + \frac{2L^2\lambda^2\eta^2 \gamma^2}{1 - L^2\lambda^2}\right)B_R(x^*_n\| x_n) + \zeta_n \\
	& = \left(\frac{1 + L^2\lambda^2}{1 - L^2 \lambda^2}\right) \left(1 - 2\alpha \eta L^{-1} + 2\eta^2 \gamma^2\right) B_R(x_n^*\| x_n) + \zeta_n .
	\end{align*}
	Thus, we have arrived at an inequality that resembles a contraction.
	However, the stepsize $\eta > 0$ may be chosen such that it minimizes the factor in front of the Bregman divergence.
	This can be achieved, but it requires that additional constraints are put on the value of $\lambda$.

	\begin{lemma}\label{lm:predictable md eta contraction}
		If $\lambda < \frac{\alpha}{2L^2\gamma}$ and $\eta = \frac{\alpha}{2L\gamma^2}$, then
		\begin{align*}
		\left(\frac{1 + L^2\lambda^2}{1 - L^2 \lambda^2}\right) \left(1 - 2\alpha \eta L^{-1} + 2\eta^2 \gamma^2\right) < 1
		\end{align*}
	\end{lemma}
	\begin{proof}

		By optimizing over choices of $\eta$, it can be seen that
		\begin{align*}
		1 - 2\alpha \eta L^{-1} + 2\eta^2 \gamma^2 \geq 1 - \frac{\alpha^2}{2L^2\gamma^2},
		\end{align*}
		where $\eta$ is chosen to be $\frac{\alpha}{2L\gamma^2}$. Therefore, in order to realize a contraction, we must have
		\begin{align*}
		1 > \left(\frac{1 + L^2\lambda^2}{1 - L^2 \lambda^2}\right) \left(1 - \frac{\alpha^2}{2L^2\gamma^2}\right).
		\end{align*}
		Alternatively,
		\begin{align*}
		0 > 2L^2\lambda^2 - \frac{\alpha^2}{2L^2\gamma^2} - \frac{\lambda^2\alpha^2}{2\gamma^2}.
		\end{align*}
		The quantity on the right hand size of the above inequality is in fact smaller than $2L^2\lambda^2 - \frac{\alpha^2}{2L^2\gamma^2}$, meaning that it is sufficient to have the condition for a contraction be:
		$
		\frac{\alpha}{2 L^2 \gamma} > \lambda.
		$.
	\end{proof}

	Note that $\frac{\alpha}{2 L^2 \gamma } < 1$ since $L \geq 1$ and $\gamma \geq \alpha$ by the definitions of smoothness of $R$ and $l_n$, respectively.
	Thus, this condition required to guarantee the contraction is stricter than requiring that $\lambda < 1$. If this condition is satisfied and if we set $\eta = \frac{\alpha}{2L\gamma^2}$, then we can further examine the contraction in terms of constants that depend only on the properties of $l_n$ and $R$:
	\begin{align*}
	B_R(x_{n+1}^*\| x_{n+1}) & \leq  \left(\frac{1 + L^2\lambda^2}{1 - L^2 \lambda^2}\right) \left(1 - 2\alpha \eta L^{-1} + 2\eta^2 \gamma^2\right) B_R(x_n^*\| x_n) + \zeta_n  \\
	& < \left( \frac{1 + \frac{\alpha^2}{4L^2\gamma^2}}{1 - \frac{\alpha^2}{4L^2\gamma^2}} \right) \left( 1 -\frac{\alpha^2}{2L^2\gamma^2} \right)B_R(x_n^*\| x_n) + \zeta_n  \\
	& = \left( 1 - \frac{\frac{\alpha^4}{8L^4\gamma^4}}{1 - \frac{\alpha^2}{4L^2\gamma^2}} \right)B_R(x_n^*\| x_n) + \zeta_n.
	\end{align*}
	It is easily verified that the factor in front of the Bregman divergence on the right side is less than $1$ and greater than $\frac{5}{6}$.

	By applying the above inequality recursively, we can derive the inequality below
	\begin{align*}
	\frac{1}{2} \| x_n - x_n^*\|^2 \leq B_R(x_n^*\| x_n) \leq  \rho^{n-1} B_R(x_1^*\| x_1) + \sum_{k=1}^{n-1} \rho^{n - k-1} \zeta_k,
	\end{align*}
	where $\rho = \left(\frac{1 + L^2\lambda^2}{1 - L^2 \lambda^2}\right) \left(1 - 2\alpha \eta L^{-1} + 2\eta^2 \gamma^2\right) < 1$. Therefore the dynamic regret can be bounded as
	\begin{align*}
	\regret_N^d & = \sum_{n = 1}^N f_n(x_n) - f_n(x_n^*) \leq G \sum_{n = 1}^N \|x_n - x_n^*\| \\
	& \leq \sqrt 2 G B_R(x_1^*\| x_1)^{1/2}\sum_{n = 1}^N\rho^{\frac{n - 1}{2}} +  \sqrt 2 G \sum_{n = 2}^N \left(\sum_{k = 1}^{n - 1} \rho^{n - k - 1} \zeta_k\right)^{1/2}  \\
	& \leq \sqrt 2 G B_R(x_1^*\| x_1)^{1/2}\sum_{n = 1}^N\rho^{\frac{n - 1}{2}} +  \sqrt 2 G\sum_{n = 2}^N\sum_{k = 1}^{n - 1} \rho^{\frac{n - k - 1}{2}} \zeta_k^{1/2},
	%\leq \sqrt 2 G B_R(x_1^*, x_1)^{1/2} \sum_{n = 1}^N\rho^{\frac{n - 1}{2}} = O(1).
	\end{align*}
	where both inequalities use the fact that for $a, b > 0$, $a + b \leq a + b + 2\sqrt{ab} = (\sqrt a + \sqrt b)^2$. The left-hand term is clearly bounded above by a constant since $\sqrt \rho < 1$.
	Analysis of the right-hand term is not as obvious, so we establish the following lemma independently.

	\begin{lemma}\label{lm:predictable md geometric series}
		If $\rho < 1$ and $\zeta_n = \frac{a_n L D_\XX}{\alpha} \left(1 + \lambda \right) + \frac{a_n^2 L}{2\alpha^2}$, then it holds that
		\begin{align*}
		\sqrt 2\sum_{n = 2}^N\sum_{k = 1}^{n - 1} \rho^{\frac{n - k - 1}{2}} \zeta_k^{1/2} = O(A_N + \sqrt {N A_N}).
		\end{align*}
	\end{lemma}
	\begin{proof}
	\begin{align*}
	\sum_{n = 2}^{N}\sum_{k = 1}^{n - 1} \rho^{\frac{n - k - 1}{2}} \zeta_k^{1/2} & = \sum_{n = 1}^{N - 1}\zeta^{1/2}_n \left( 1 + \rho^{\frac{1}{2}} + \ldots + \rho^{\frac{N - 1 - n}{2}} \right)
	 \leq \frac{1}{1 - \sqrt \rho} \sum_{n = 1}^{N - 1}\sqrt{\zeta_n}.
	\end{align*}
	The last inequality upper bounds the finite geometric series with the value of the infinite geometric series since again $\sqrt \rho < 1$ for each $k$. Recall that $\zeta_n$ was defined as
	\begin{align*}
	\zeta_n = \frac{a_n L D_\XX}{\alpha} \left(1 + \lambda \right) + \frac{a_n^2 L}{2\alpha^2}.
	\end{align*}
	Therefore, the over the square roots can be bounded:
	\begin{align*}
	\sum_{n = 1}^{N - 1}\sqrt{\zeta_n} & \leq \sqrt{\frac{ L D_\XX}{\alpha} \left(1 + \lambda \right)}\sum_{n = 1}^{N-1} \sqrt{ a_n} + \alpha^{-1}\sqrt{\frac{L}{2}} \sum_{n = 1}^{N-1} a_n.
	\end{align*}
	While the right-hand summation is simply the definition of $A_{N-1}$, the left-hand summation yields $\sum_{n =1}^{N - 1}\sqrt{a_n} \leq \sqrt {(N - 1)A_{N -1}}$.
	\end{proof}

	Then the total dynamic regret has order $\regret^d_N = O(1 + A_N + \sqrt {NA_N})$.
%\end{proof}

\subsection{Proof of~\cref{th:alpha equals beta predictable problem}}
%\begin{proof}[Proof of~\cref{th:alpha equals beta predictable problem}]

	\subsubsection{Euclidean Space with $\frac{\beta}{\alpha}=1$}

	The proof first requires a result from analysis on the convergence of sequences that are nearly monotonic.
	\begin{lemma}\label{lm:general monotone convergence theorem}
		Let $(a_n)_{n \in \N} \subset \R$ and $(b_n)_{n \in \N} \subset \R$ be two sequences satisfying $b_n \geq 0$ and $\sum_{k = 1}^n a_k < \infty$ $\forall n \in \N$. If $b_{n+1} \leq b_n + a_n$, then the sequence $b_n$ converges.
	\end{lemma}
	\begin{proof}
		Define $u_1 := b_1$ and $u_n := b_n - \sum_{k = 1}^{n - 1} a_k$. Note that $u_1 = b_1 \geq b_2 - a_1 = u_2$. Recursively, $b_n - a_{n-1} \leq b_{n-1} \implies b_n - \sum_{k = 1}^{n-1} a_k \leq b_{n-1} - \sum_{k = 1}^{n-2}$. Therefore, $u_n \leq u_{n+1}$. Note that $(u_n)_{n \in \N})$ is bounded below because $b_n \geq 0$ and $\sum_{k = 1}^n a_k <\infty$. This implies that $(u_n)_{n \in \N}$ converges. Because $\left(\sum_{k = 1}^n a_k \right)_{n \in \N}$, also converges, $(b_n)_{n \in \N}$ must converge.
	\end{proof}
	The majority of the proof follows a similar line of reasoning as a standard result in the field of discrete-time pursuit-evasion games \cite{alexander2006pursuit}.
	Let $\|\cdot\|$ denote the Euclidean norm.
	We aim to show that if the distance between the learner's decision $x_n$ and the optimal decision $x_n^*$ does not converge to zero, then they travel unbounded in a straight line, which is a contradiction.

	Consider the following update rule which essentially amounts to a constrained greedy update:
	\begin{align*}
		x_{n+1} = \frac{x_n + x_n^*}{2}
	\end{align*}
%
%	Consider a modified version of the greedy algorithm, which has the following update rule:
%	\begin{align*}
%	& x_{n+1}  =  \argmin_{x \in \XX}\ l_n(x) \\ & \text{s.t.} \quad \|x - x_n\| \leq \frac{1}{2} \|x_n  - x_n^*\|
%	\end{align*}
%	That is, $x_n$ is constrained such that it can only move half the distance to $x_n^*$ in each round.
%	\cheng{It is unclear how $\|x_{n+1} - x_{n}^*\|  =  \frac{1}{2}\|x_{n} - x_{n}^*\|$ follows from the above greedy update	;
%	Would it be more transparent to use the update rule below?
%	\begin{align*}
%	x_{n+1} = \frac{x_n + x_n^*}{2}
%	\end{align*}
%	}
	$x_{n+1}$ is well defined at each round because $\XX$ is convex. Define $c_n := \|x_{n} - x_n^*\|$.
	Then we have
	\begin{align*}
	0 \leq c_{n+1} & = \|x_{n+1} - x_{n+1}^*\| \\
	& \leq \|x_{n+1} - x_{n}^*\| + \|x_{n+1}^* - x_{n}^*\| \\
	& = \frac{1}{2}\|x_{n} - x_{n}^*\| + \|x_{n+1}^* - x_{n}^*\| \\
	&\leq \frac{1}{2}\|x_{n} - x_{n}^*\| + \|x_{n+1} - x_{n}\| + \frac{a_n}{\alpha} \qquad (\because \text{\cref{pr:generalized contraction property}})\\
	& = \|x_{n} - x_{n}^*\| + \frac{a_n}{\alpha} = c_n + \frac{a_n}{\alpha} \qquad %\text{\cheng{$\leq$ would become $=$}}
	\end{align*}
	Because it is assumed that $\sum_{n = 1}^\infty a_n < \infty$, the sequences $(c_n)_{n \in \N}$ and $(a_n)_{n \in \N}$ satisfy the sufficient conditions of~\cref{lm:general monotone convergence theorem}.
	Thus the sequence $(c_n)_{n \in \N}$ converges, so there exists a limit point $C := \lim_{n \rightarrow \infty} c_n \geq 0$.
	Towards a contradiction, consider the case where $C > 0$. We will prove that this leads the points to follow a straight line in the following lemma.

	\begin{lemma}\label{lm:pursuit-evasion straight line}
		Let $\theta_n$ denote the angle between the vectors from $x_n^*$ to $x_{n+1}^*$ and from $x_n^*$ to $x_{n+1}$. If $\lim_{n \rightarrow \infty} c_n > 0$, then $\lim_{n \rightarrow\infty} \cos \theta_n = -1$.
	\end{lemma}
	\begin{proof}
		At round $n+1$ we can write the distance between the learner's decision and the optimal decision in terms of the previous round:
			\begin{align*}
			C^2  & = \lim_{n \rightarrow \infty}\|x_{n+1} - x_{n+1}^*\|^2 \\
			& = \lim_{n \rightarrow \infty} \left( \|x_{n+1} -  x_n^*\|^2 + \|x_{n+1}^* -  x_{n}^*\|^2 - 2 \|x_{n+1} -  x_n^*\|\|x_{n+1}^* -  x_{n}^*\| \cos \theta_n \right)\\
			& \leq \lim_{n \rightarrow \infty} \left( \frac{1}{4}\|x_{n} -  x_n^*\|^2 + \|x_n - x_{n+1}\|^2 + \frac{a_n^2}{\alpha^2} + \frac{2a_n }{\alpha} \|x_n - x_{n+1}\| - 2 \|x_{n+1} -  x_n^*\|\|x_{n+1}^* -  x_{n}^*\| \cos \theta_n \right) \\
			& = \lim_{n \rightarrow \infty} \left( \frac{1}{2}\|x_{n} -  x_n^*\|^2 + \frac{a_n^2}{\alpha^2} + \frac{2a_n }{\alpha} \|x_n - x_{n+1}\| - 2 \|x_{n+1} -  x_n^*\|\|x_{n+1}^* -  x_{n}^*\| \cos \theta_n \right) \\
			& =  \lim_{n \rightarrow \infty} \frac{1}{2}\|x_{n} -  x_n^*\|^2  - 2\lim_{n \rightarrow \infty} \|x_{n+1} -  x_n^*\|\|x_{n+1}^* -  x_{n}^*\| \cos \theta_n\\
			& = \frac{1}{2} C^2  - 2\lim_{n \rightarrow \infty} \|x_{n+1} -  x_n^*\|\|x_{n+1}^* -  x_{n}^*\| \cos \theta_n
			\end{align*}
	The first inequality follows because $\|x_{n+1} - x_n^*\| = \frac{1}{2}\|x_n - x_n^*\|$ and $\|x_{n+1}^* - x_n^*\| \leq \|x_{n+1} - x_n\| + \frac{a_n}{\alpha}$ due to~\cref{pr:generalized contraction property}. The next equality again uses $\|x_{n+1} - x_n^*\| = \frac{1}{2}\|x_n - x_n^*\|$. The second to last line follows from passing the limit through the sum, where we have $\lim_{n \rightarrow \infty} a_n = 0$ because $A_\infty < \infty$.
	That is, the inequality above implies
	\begin{align*}
	 2\lim_{n \rightarrow \infty} \|x_{n+1} -  x_n^*\|\|x_{n+1}^* -  x_{n}^*\| \cos \theta_n = -\frac{C^2}{2} < 0
	\end{align*}
	which in turn implies $\lim_{n \rightarrow \infty} \cos\theta_n < 0$.
	This leads to an upper bound
	\begin{align*}
		- 2\lim_{n \to \infty} \|x_{n+1} -  x_n^*\|\|x_{n+1}^* -  x_{n}^*\| \cos \theta_n
	 &=    \left(-2 \lim_{n \to \infty} \cos  \theta_n \right) \lim_{n \to \infty} \|x_{n+1} -  x_n^*\|\|x_{n+1}^* -  x_{n}^*\| \\
	 &\leq \left(-2 \lim_{n \to \infty} \cos  \theta_n \right) \lim_{n \to \infty} \frac{1}{2}\|x_{n} -  x_n^*\| \left( \|x_{n+1} -  x_{n}\| + \frac{a_n}{\alpha} \right)\\
	 &= \frac{-C^2}{2} \lim_{n \to \infty} \cos  \theta_n
	\end{align*}
	Combining these two inequalities, we can then conclude $C^2 \leq\frac{C^2}{2} - \frac{C^2}{2}\cos \theta \leq C^2$. A necessary condition in order for the bounds to be satisfied is $\cos \theta = -1$.
\end{proof}

%	Let $\theta_n$ denote the angle between the vectors from $x_n^*$ to $x_{n+1}^*$ and from $x_n^*$ to $x_{n+1}$.
%	We intend to show that $\cos \theta_n \rightarrow -1$. By the law of cosines,
%	\begin{align*}
%	C^2 = \lim_{n \rightarrow \infty}\|x_{n+1} - x_{n+1}^*\|^2 & = \lim_{n \rightarrow \infty} \left( \|x_{n+1} -  x_n^*\|^2 + \|x_{n+1}^* -  x_{n}^*\|^2 - 2 \|x_{n+1} -  x_n^*\|\|x_{n+1}^* -  x_{n}^*\| \cos \theta_n \right)\\
%	& \leq \lim_{n \rightarrow \infty} \left( \frac{1}{2}\|x_{n} -  x_n^*\|^2 + \frac{a_n^2}{\alpha^2} + \frac{2a_n \beta}{\alpha^2} \|x_n - x_{n+1}\| - 2 \|x_{n+1} -  x_n^*\|\|x_{n+1}^* -  x_{n}^*\| \cos \theta_n \right) \\
%	& \leq  \lim_{n \rightarrow \infty} \frac{1}{2}\|x_{n} -  x_n^*\|^2  - 2\lim_{n \rightarrow \infty} \|x_{n+1} -  x_n^*\|\|x_{n+1}^* -  x_{n}^*\| \cos \theta_n
%	\end{align*}

	When $C>0$, \cref{lm:pursuit-evasion straight line} therefore implies the points $x_{n}, x_{n+1}, x_n^*, x_{n+1}^*$ are colinear in the limit. Thus, $\|x_{n} - x_{n + m}\|$ grows unbounded in $m$, which contradicts the compactness of $\XX$.
	The alternative case must then be true: $C = \lim_{n \rightarrow\infty}\|x_n - x_n^*\| = 0$.
	The dynamic regret can then be bounded as:
	\begin{align*}
	\regret_N^d = \sum_{n = 1}^N l_n(x_n) - l_n(x_n^*) \leq G\sum_{n = 1}^N \|x_n - x_n^*\|
	\end{align*}
	Since $\|x_N - x_N^*\| \rightarrow 0$, we know $\lim_{N \rightarrow \infty} \frac{1}{N}\sum_{n = 1}^N \| x_n - x_n^*\| = 0$. Therefore, the dynamic regret is sublinear.

	Note that this result does not reveal a rate of convergence, only that $\|x_n - x_n^*\|$ converges to zero, which is enough for sublinear dynamic regret.

	\subsubsection{One-dimensional Space with arbitrary $\frac{\beta}{\alpha}$}

	In the case where $d = 1$, we aim to prove sublinear dynamic regret regardless of $\alpha$ and $\beta$ by showing that $x_n$ essentially traps $x_n^*$ by taking conservative steps as before. Rather than the constraint being $|x_n - x_{n+1}| \leq \frac{1}{2}|x_n - x_{n}^*|$, we choose $x_{n+1}$ in the direction of $x_n^*$ subject to $|x_n - x_{n+1}| \leq \frac{1}{1 + \lambda} |x_n - x_n^*|$.
	Specifically, we will use the following update rule:
	%	\begin{align*}
	%	& x_{n+1} = \argmin_x \|x - x_n^*\| \\
	%	& \text{s.t.} \quad \|x - x_n \| \leq \frac{\|x_n - x_n^* \|}{2(1 + \lambda)}.
	%	\end{align*}
	\begin{align} \label{eq:lambda avg update}
	x_{n+1} = \frac{\lambda  x_n +  x_n^* }{1+\lambda}
	\end{align}
	Recall that sublinear dynamic regret is implied by  $c_n := |x_n - x_n^*|$ converging to zero as $n\to\infty$. Therefore, below we will prove the above update rule results in $\lim_{n\to\infty} c_n = 0$. Like our discussions above, this implies achieving sublinear dynamic regret but not directly its rate.

		Suppose at any time $|x_n - x_n^*| = 0$. Then we are done since the learner can repeated play the same decision without $x_n^*$ changing. Below we consider the case $|x_n - x_n^*| \neq 0$. We prove this by contradiction. First we observe that the update in~\eqref{eq:lambda avg update} makes sure that, at any round, $x_{n+1}^*$ cannot switch to the opposite side of $x_n^*$ with respect to $x_{n+1}$ and $x_n$; namely it is guaranteed that $(x_{n+1}^* - x_{n+1})(x_n^* - x_{n+1}) \geq 0$ and  $(x_{n+1}^* - x_{n})(x_n^* - x_{n}) \geq 0$.

		Towards a contradiction, suppose that there is some $C > 0$ such that $|x_n - x_n^* | \geq C$ for infinitely many $n$. Then $x_n$ at every round moves a distance of at least $\frac{C}{1 + \lambda}$ in the same direction infinitely since $x_{n+1}^*$ always lies the same side of $x_{n+1}$ as $x_n^*$. This contradicts the compactness of $\XX$. Therefore $|x_n - x_n^* |$ must converge to zero.

\section{New Insights into Imitation Learning}\label{sec:IL}

In this section, we investigate an application of the COL framework in the sequential decision problem of online IL~\citep{ross2011reduction}.
We consider an episodic MDP with state space $\SS$, action space $\AA$, and finite horizon $H$.
For any $s, s' \in \SS$ and $a \in \AA$, the transition dynamics $\PP$ gives the conditional density, denoted by $\PP(s' | s, a)$, of transitioning to $s'$ starting from state $s$ and applying action $a$. The reward of state $s$ and action $a$ is denoted as $r(s, a)$.
A deterministic policy $\pi$ is a mapping from $\SS$ to a density over $\AA$. %, denoted by $\Delta(\AA)$.
We suppose the MDP starts from some fixed initial state distribution. We denote the probability of being in state $s$ at time $t$ under policy $\pi$ as $d_t^\pi (s)$, and we define the average state distribution under $\pi$ as $d^\pi(s) = \frac{1}{T} \sum_{t = 1}^T d_t^\pi(s)$.
%The objectives and decision variables are problem-specific,
%and we will discuss how they align with the COL framework in the section.

In IL, we assume that $\PP$ and $r$ are unknown to the learner, but, during training time, the learner is given access to an expert policy $\pi^\star$ and full knowledge of a supervised learning loss function $c(s, \pi; \pi^\star)$, defined for each state $s \in \SS$.
The objective of IL is to solve
\begin{align} \label{eq:IL objective}
\min_{\pi \in \Pi} \quad \E_{s \sim d^\pi} \left[ c(s, \pi; \pi^\star) \right],
\end{align}
where $\Pi$ is the set of allowable parametric policies, which will be assumed to be convex. %$\subseteq \R^d$
Note that it is often the case that $\pi^\star \not \in  \Pi$.

As $d^\pi$ is not known analytically, optimizing \eqref{eq:IL objective} directly leads to a reinforcement learning problem and therefore can be sample inefficient. \emph{Online IL}, such as the popular \textsc{DAgger} algorithm~\cite{ross2011reduction}, bypasses this difficulty by reducing \eqref{eq:IL objective} into a sequence of supervised learning problems.
Below we describe a general construction of online IL: at the $n$th iteration
(1) execute the learner's current policy $\pi_n$ in the MDP to collect state-action samples;
(2) update $\pi_{n+1}$ with information of the stochastic approximation of $l_n(\pi) = \E_{d^{\pi_n}}\left[ c(s, \pi; \pi^\star) \right]$ based the samples collected in the first step. Importantly, we remark that in these empirical risks, the states are sampled according to $d^{\pi_n}$ of the learner's policy.

The use of online learning to analyze online IL is well established \citep{ross2011reduction}.
As studied in \cite{cheng2018convergence,lee2018dynamic}, these online losses can be formulated as a bifunction,
$l_n(\pi) = f_{\pi_n}(\pi) = \E_{s \sim d^{\pi_n}} \left[ c(s, \pi; \pi^\star)\right]$, and the policy class $\Pi$ can be viewed as the decision set $\XX$.
Naturally, this online learning formulation results in many online IL algorithms resembling standard online learning algorithms, such as follow-the-leader (FTL), which uses full information feedback $l_n(\cdot) =\E_{s \sim  d^{\pi_n}} \left[ c(s, \cdot; \pi^\star)\right]$ at each round \citep{ross2011reduction},
and mirror descent \citep{sun2017deeply}, which uses the first-order feedback $\nabla l_n(\pi_n) = \E_{d^{\pi_n}} \left[ \nabla_{\pi_n} c(s, \pi_n; \pi^\star)\right]$. This feedback can also be approximated by unbiased samples.
The original work by \cite{ross2011reduction} analyzed FTL in the static regret case by immediate reductions to known static regret bounds of FTL.
However, a crucial objective is understanding when these algorithms converge to useful solutions in terms of policy performance, which more recent work has attempted to address \citep{cheng2018convergence, lee2018dynamic, cheng2018accelerating}. According to these refined analyses, dynamic regret is a more appropriate solution concept to online IL when $\pi^\* \notin \Pi$, which is the common case in practice.

%Given an MDP with an unknown transition distribution and a supervisor policy $\pi^\star$, the objective in IL is to find a parameterized policy $\pi \in \XX \subset \R^d$ that minimizes a given supervised learning loss function $c(s, \pi ; \pi^\star)$ along rollouts generated by the learner's policy $\pi$.
%That is, the objective is the risk $\E_{s \sim d_{\pi}} \left[ c(s, \pi; \pi^\star ) \right]$. Since $d^\pi$ is not known analytically, this problem is traditionally solved iteratively by observing risks based on the current policy's distribution $f_{\pi_n}(\pi) = \E_{s \sim d^{\pi_n}} \left[ c(s, \pi; \pi^\star ) \right]$ and using sampled feedback to update $\pi_{n+1}$.

Below we frame online IL in the proposed COL framework and study its properties based on the properties of COL that we obtained in the previous sections.
We have already shown that the per-round loss $l_n(\cdot)$ can be written as the evaluation of a bifunction $f_{\pi_n}(\cdot)$. This COL problem is actually an $(\alpha, \beta)$-regular COL problem when the expected supervised learning loss $\E_{s\sim d^{\pi_n}}[c(s, \pi;\pi^\star)]$ is strongly convex in $\pi$ and the state distribution $d^\pi$ is Lipschitz continuous (see \cite{ross2011reduction,cheng2018convergence,lee2018dynamic}). We can then leverage our results in the COL framework to immediately answer an interesting question in the online IL problem.
%As in \cite{ross2011reduction,cheng2018convergence,lee2018dynamic}, we strengthen the continuity and convexity assumptions with $\beta$-Lipschitz continuity of $\nabla f_\cdot(x)$ and $\alpha$-strong convexity of $f_x(\cdot)$, which corresponds to an $(\alpha, \beta)$-regular problem.
\begin{proposition}
	When $\alpha > \beta$, there exists a \emph{unique} policy $\widehat \pi$ that is optimal on its own distribution:
	\begin{align*}
	\E_{s \sim d_{\widehat \pi_n}} \left[ c(s, \widehat \pi; \pi^\star ) \right] =
	\min_{\pi\in\Pi}  \E_{s \sim d_{\widehat \pi_n}} \left[ c(s, \pi; \pi^\star ) \right].
	\end{align*}
\end{proposition}
This result is immediate from the fact that $\alpha > \beta$ implies that $\nabla f_\pi(\pi)$ is a $\mu$-strongly monotone VI with $\mu = \beta - \alpha$ by~\cref{pr:beta-alpha strongly monotone}. The VI is therefore guaranteed to have a unique solution \citep{facchinei2007finite}.

Furthermore, we can improve upon the known conditions sufficient to find this policy through online gradient descent and give a non-asymptotic convergence guarantee through a reduction to strongly monotone VIs.
%\cheng{Is the following condition correct? Isn't this already covered by $\beta$ regularity.}
We will additionally assume that $f$ is $\gamma$-smooth in $\pi$, satisfying $\| \nabla f_{\pi'}(\pi_1) - \nabla f_{\pi'} (\pi_2) \| \leq \gamma \| \pi_1 - \pi_2\|$ for any fixed query argument $\pi'$.
%
%As in \citep{lee2018dynamic}, we assume that $d^\pi$ satisfies $\delta(d^\pi, d_{\pi'}) \leq \beta_\delta \| \pi - \pi'\|$ where $\delta: \mathbb P(s) \times \mathbb P(s) \to \R_+$ is the total variation distance.

%This implies $f_{\pi}(\pi')$ is $(\alpha, \beta)$-regular as long as $\|\nabla c\|_\infty = \sup_{\tau, \pi} \|\nabla c(\tau, \pi; \pi^\star)\|$ is finite.
%When $\alpha > \beta$, by \cref{th:equivalent problems} and \cref{pr:beta-alpha strongly monotone}, the VI with $F(\pi) = \E_{s \sim d_{\pi}} \left[ \nabla c(s, \pi;\pi^\star) \right]$ is $\mu$-strongly monotone with $\mu = \alpha - \beta$.

%We then apply the projection algorithm \citep{facchinei2007finite}, which is equivalent to online gradient descent studied in \citep{sun2017deeply,lee2018dynamic}. Let $P_\Pi$ denote the Euclidean projection onto $\Pi$. The online gradient descent algorithm can be described as computing the following at each round:

We then apply our results from \cref{sec:algorithms}. Specifically, we consider mirror descent with $B_R(\pi \| \pi') = \frac{1}{2}\|\pi - \pi'\|_2^2$, which is equivalent to online gradient descent studied in \cite{sun2017deeply,lee2018dynamic}. Note that $R = \frac{1}{2} \| \pi\|_2^2$, which is $1$-strongly convex and $1$-smooth. Then, we apply \cref{lm:mirror descent contraction}.%\cref{pr:mirror descent dynamic regret}.
%Let $P_\Pi$ denote the Euclidean projection onto $\Pi$. The online gradient descent algorithm can be described as computing the following at each round:
%$
%\pi_{n+1} = P_\Pi ( \pi_n - \eta_n g_n )
%$ or equivalently
%\begin{align*}
%\pi_{n+1} = \argmin_{\pi \in \Pi } \ \eta_n \< g_n, \pi\> + \frac{1}{2} \| \pi - \pi_n\|^2.
%\end{align*}
%
%\begin{proposition}[\cite{facchinei2007finite}]
%	If $\alpha > \beta$ and  the stepsize is chosen such that $\eta = \frac{2\mu}{(L  + \beta)^2}$, then, under the online gradient descent algorithm with deterministic feedback $g_n = \nabla l_n(\pi_n)$, it holds that
%	\begin{align*}
%	\| \pi_n  - \widehat \pi \|^2 \leq \left( 1 + (L + \beta)^2\eta^2 - 2\mu\eta \right)^{n - 1} \| \pi_1 - \widehat \pi\|^2
%	\end{align*}
%\end{proposition}

\begin{corollary}
	If $\alpha > \beta$ and  the stepsize is chosen such that $\eta = \frac{\alpha - \beta}{(\gamma  + \beta)^2}$, then, under the online gradient descent algorithm with deterministic feedback $g_n = \nabla l_n(\pi_n)$, it holds that
	\begin{align*}
	\| \pi_n  - \widehat \pi \|^2 \leq \left( 1 -  \left(\frac{\alpha - \beta}{\gamma + \beta}\right)^2\right)^{n - 1} \| \pi_1 - \widehat \pi\|^2
	\end{align*}
\end{corollary}

By \cref{pr:mirror descent dynamic regret}, $\regret_N^d$ will therefore be sublinear (in fact, $\regret_N^d = O(1)$) and the policy converges linearly to the policy that is optimal on its own distribution, $\widehat \pi$. The only condition required on the problem itself is $\alpha > \beta$ while the state-of-the-art sufficient condition of \cite{lee2018dynamic} additionally requires $\frac{\alpha}{\gamma} > \frac{2\beta}{\alpha}$.
The result also gives a new non-asymptotic convergence rate to $\widehat \pi$.

The above result only considers the case when the feedback is deterministic; i.e., there is no sampling error due to executing the policy on the MDP, and the risk $\E_{d^{\pi_n}}\left[ c(s, \pi; \pi^\star) \right]$ is known exactly at each round. While this is a standard starting point in analysis of online IL algorithms \citep{ross2011reduction}, we are also interested in the more realistic stochastic case, which has so far not been analyzed for the online gradient descent algorithm in online IL. It turns out that the COL framework can be easily leveraged here too to provide a sublinear dynamic regret bound.

At round $n$, we consider observing the empirical risk $\tilde l_n (\pi) =  \frac{1}{T} \sum_{t = 1}^{T} c(s_t, \pi; \pi^\star)$ where $s_t \sim d_t^{\pi_n}$. Note that $\E [ \tilde l_n (\pi) | \pi_n ] = l_n(\pi)$ and it is easy to show that the first-order feedback $\nabla\tilde{l}_n(\pi_n)$ can be modeled as the expected gradient with an additive zero-mean noise: $g_n = \nabla l_n(\pi_n) + \epsilon_n$. For simplicity, we assume $\E\left[ \| \epsilon_n \|^2 \right] < \infty$.

\begin{corollary}
	If $\alpha > \beta$ and the stepsize is chosen as $\eta_n = \frac{1}{\sqrt n}$, then, under online gradient descent with stochastic feedback, it holds that $\E[\regret_N^d] = O(\sqrt N)$.
\end{corollary}
This corollary follows from \cref{pr:stochastic mirror descent}, which in turn leverages the reduction to static regret in~\cref{cr:full reduction to static regret}. %The dynamic regret is worse than that of the deterministic case, but it is still sublinear. This is the price paid for stochastically sampling from the MDP.

\end{document}